\documentclass[10pt,twocolumn,letterpaper]{article}

\usepackage[pagenumbers]{wacv}

\usepackage{tikz}
\definecolor{GrayHighlight}{gray}{0.9}
\usetikzlibrary{arrows.meta,positioning}

\usepackage{todonotes} 

\usepackage[suppress]{color-edits}
\addauthor[Ofir]{ofir}{blue}
\addauthor[agents]{agents}{red}

\usepackage{multirow}

\usepackage{placeins} 

\newcommand{\vismix}{VisionMX}
\newcommand{\vismixLW}{\vismix-LW}
\newcommand{\vismixBW}{\vismix-BW}
\newcommand{\qdrop}{QDrop}

\usepackage{ifthen}
\newboolean{edit}
\setboolean{edit}{true}
\newcommand{\editenable}[1]{\setboolean{edit}{true}}
\newcommand{\editdisable}[1]{\setboolean{edit}{false}}

\usepackage{amsthm}
\usepackage{amssymb}
\usepackage{algorithm}
\usepackage{amsmath}
\usepackage{acronym}

\usepackage{commath,amsmath,amssymb,amsfonts}

\usepackage{accents}
\usepackage[inkscapearea=page]{svg}

\usepackage{mathtools}
\usepackage{amssymb}
\usepackage{bm}
\usepackage{graphicx}
\usepackage{float}
\usepackage{amsmath}
\usepackage{tikz}
\usepackage{multicol}
\usepackage{dsfont}
\usepackage{tcolorbox}
\usepackage{bbm}
\usepackage{multirow}
\usetikzlibrary{positioning, backgrounds, fit, shapes.arrows}
\usetikzlibrary{decorations.markings}
\usepackage{cuted}
\usepackage{enumitem}
\usepackage{caption}
\usepackage{booktabs}
\usepackage{subcaption}
\newtheorem{theorem}{Theorem}[section]
\newtheorem{corollary}[theorem]{Corollary}

\newtheorem{prop}[theorem]{Proposition}

\newtheorem{assumption}{Assumption}[section]

\usepackage{algpseudocode}
\tikzset{block/.style = {draw, fill=white, rectangle,
		minimum height=3em, minimum width=2cm},
	input/.style = {coordinate},
	output/.style = {coordinate},
	pinstyle/.style = {pin edge={to-,t,black}}
	radiation/.style={{decorate,decoration={expanding waves,angle=90,segment   length=4pt}}}
	
}
\usepackage{smartdiagram}

\tikzstyle{block} = [draw, rectangle, minimum height=2em, minimum width=2em]
\tikzstyle{sum} = [draw, circle,minimum width=0.1 cm]
\tikzstyle{input} = [coordinate]
\tikzstyle{output} = [coordinate]
\tikzstyle{dummy} = [coordinate]
\tikzstyle{pinstyle} = [pin edge={to-,thin,black}]
\usetikzlibrary{positioning, fit, arrows.meta}
\usetikzlibrary{positioning}
\usetikzlibrary{shapes,arrows}
\tikzstyle{frame_cyan} = [thick, draw=blue, solid,inner sep=0.3em]
\tikzstyle{frame_red} = [thick, draw=red, solid,inner sep=0.3em]
\tikzstyle{frame_green} = [thick, draw=green, solid,inner sep=0.3em]

\usepackage{tikz}
\usepackage{xcolor}
\definecolor{fc}{HTML}{1E90FF}
\definecolor{h}{HTML}{228B22}
\definecolor{bias}{HTML}{87CEFA}
\definecolor{noise}{HTML}{8B008B}
\definecolor{conv}{HTML}{FFA500}
\definecolor{pool}{HTML}{B22222}
\definecolor{up}{HTML}{B22222}
\definecolor{view}{HTML}{FFFFFF}
\definecolor{bn}{HTML}{FFD700}
\tikzset{fc/.style={black,draw=black,fill=fc,rectangle,minimum height=1cm}}
\tikzset{h/.style={black,draw=black,fill=h,rectangle,minimum height=1cm}}
\tikzset{bias/.style={black,draw=black,fill=bias,rectangle,minimum height=1cm}}
\tikzset{noise/.style={black,draw=black,fill=noise,rectangle,minimum height=1cm}}
\tikzset{conv/.style={black,draw=black,fill=conv,rectangle,minimum height=1cm}}
\tikzset{pool/.style={black,draw=black,fill=pool,rectangle,minimum height=1cm}}
\tikzset{up/.style={black,draw=black,fill=up,rectangle,minimum height=1cm}}
\tikzset{view/.style={black,draw=black,fill=view,rectangle,minimum height=1cm}}
\tikzset{bn/.style={black,draw=black,fill=bn,rectangle,minimum height=1cm}}
 
\usepackage{xspace}

\tikzstyle{dummy} = [coordinate]
\pgfkeys{/pgf/.cd,
  parallelepiped offset x/.initial=2mm,
  parallelepiped offset y/.initial=2mm
}
\pgfdeclareshape{parallelepiped}
{
  \inheritsavedanchors[from=rectangle] 
  \inheritanchorborder[from=rectangle]
  \inheritanchor[from=rectangle]{north}
  \inheritanchor[from=rectangle]{north west}
  \inheritanchor[from=rectangle]{north east}
  \inheritanchor[from=rectangle]{center}
  \inheritanchor[from=rectangle]{west}
  \inheritanchor[from=rectangle]{east}
  \inheritanchor[from=rectangle]{mid}
  \inheritanchor[from=rectangle]{mid west}
  \inheritanchor[from=rectangle]{mid east}
  \inheritanchor[from=rectangle]{base}
  \inheritanchor[from=rectangle]{base west}
  \inheritanchor[from=rectangle]{base east}
  \inheritanchor[from=rectangle]{south}
  \inheritanchor[from=rectangle]{south west}
  \inheritanchor[from=rectangle]{south east}
  \backgroundpath{
    \southwest \pgf@xa=\pgf@x \pgf@ya=\pgf@y
    \northeast \pgf@xb=\pgf@x \pgf@yb=\pgf@y
    \pgfmathsetlength\pgfutil@tempdima{\pgfkeysvalueof{/pgf/parallelepiped offset x}}
    \pgfmathsetlength\pgfutil@tempdimb{\pgfkeysvalueof{/pgf/parallelepiped offset y}}
    \def\ppd@offset{\pgfpoint{\pgfutil@tempdima}{\pgfutil@tempdimb}}
    \pgfpathmoveto{\pgfqpoint{\pgf@xa}{\pgf@ya}}
    \pgfpathlineto{\pgfqpoint{\pgf@xb}{\pgf@ya}}
    \pgfpathlineto{\pgfqpoint{\pgf@xb}{\pgf@yb}}
    \pgfpathlineto{\pgfqpoint{\pgf@xa}{\pgf@yb}}
    \pgfpathclose
    \pgfpathmoveto{\pgfqpoint{\pgf@xb}{\pgf@ya}}
    \pgfpathlineto{\pgfpointadd{\pgfpoint{\pgf@xb}{\pgf@ya}}{\ppd@offset}}
    \pgfpathlineto{\pgfpointadd{\pgfpoint{\pgf@xb}{\pgf@yb}}{\ppd@offset}}
    \pgfpathlineto{\pgfpointadd{\pgfpoint{\pgf@xa}{\pgf@yb}}{\ppd@offset}}
    \pgfpathlineto{\pgfqpoint{\pgf@xa}{\pgf@yb}}
    \pgfpathmoveto{\pgfqpoint{\pgf@xb}{\pgf@yb}}
    \pgfpathlineto{\pgfpointadd{\pgfpoint{\pgf@xb}{\pgf@yb}}{\ppd@offset}}
  }
}
\pgfdeclareshape{document}{
\inheritsavedanchors[from=rectangle] 
\inheritanchorborder[from=rectangle]
\inheritanchor[from=rectangle]{center}
\inheritanchor[from=rectangle]{north}
\inheritanchor[from=rectangle]{north east}
\inheritanchor[from=rectangle]{north west}
\inheritanchor[from=rectangle]{south}
\inheritanchor[from=rectangle]{south east}
\inheritanchor[from=rectangle]{south west}
\inheritanchor[from=rectangle]{west}
\inheritanchor[from=rectangle]{east}
\backgroundpath{%
\southwest \pgf@xa=\pgf@x \pgf@ya=\pgf@y
\northeast \pgf@xb=\pgf@x \pgf@yb=\pgf@y
\pgf@xc=\pgf@xb \advance\pgf@xc by-5pt 
\pgf@yc=\pgf@ya \advance\pgf@yc by5pt
\pgfpathmoveto{\pgfpoint{\pgf@xa}{\pgf@ya}}
\pgfpathlineto{\pgfpoint{\pgf@xa}{\pgf@yb}}
\pgfpathlineto{\pgfpoint{\pgf@xb}{\pgf@yb}}
\pgfpathlineto{\pgfpoint{\pgf@xb}{\pgf@yc}}
\pgfpathlineto{\pgfpoint{\pgf@xc}{\pgf@ya}}
\pgfpathclose
\pgfpathmoveto{\pgfpoint{\pgf@xc}{\pgf@ya}}
\pgfpathlineto{\pgfpoint{\pgf@xc}{\pgf@yc}}
\pgfpathlineto{\pgfpoint{\pgf@xb}{\pgf@yc}}
\pgfpathclose
}
}
\tikzstyle{block} = [draw, fill=white, rectangle, 
    minimum height=3em, minimum width=6em]
    
\usetikzlibrary{backgrounds}
\usepackage{pifont}
\tikzstyle{startstop} = [rectangle, rounded corners, minimum width=2cm, minimum height=0.7cm,text centered, draw=black, fill=lime!30]
\usepackage{stackengine}

\makeatletter
\newcounter{phase}[algorithm]
\newlength{\phaserulewidth}
\newcommand{\setphaserulewidth}{\setlength{\phaserulewidth}}

\makeatother
\setphaserulewidth{.7pt}
\makeatother
\usepackage{colortbl}
\newcommand\MakeUppercaseGreek[1]{
  \begingroup
    \let\psi\Psi
    \let\omega\Omega
    \let\gamma\Gamma
    \MakeUppercase{#1}
  \endgroup}
\newcommand{\vectorsym}[1]{\bm{#1}}

\newcommand{\matsym}[1]{\mathbf{#1}}
\newcommand{\squareb}[1]{\left[{#1}\right]}

\newcommand{\probP}{\text{I\kern-0.15em P}}

\newcommand{\mxWeight}{\vectorsym{w}} 
\newcommand{\mxRelaxedWeight}{\widetilde{\vectorsym{w}}} 
\newcommand{\mxBlockIndex}{b} 
\newcommand{\mxScale}{s} 
\newcommand{\mxQuantizedScale}{\widehat{s}} 
\newcommand{\mxScaleQuantizer}{Q_{\mathrm{S}}} 
\newcommand{\mxScaleSet}{\mathcal{S}} 
\newcommand{\mxFormatSet}{\mathcal{Q}} 
\newcommand{\mxQuantizer}{Q_{\mathrm{MX}}} 
\newcommand{\mxElementQuantizer}{Q_{\mathrm{E}}} 
\newcommand{\mxBlockSize}{B} 
\newcommand{\mxNumBlocks}{N_{\mathrm{blk}}} 
\newcommand{\mxBlockElementSet}{\mathcal{I}} 
\newcommand{\mxMaxElementMagnitude}{q_{\max}} 
\newcommand{\mxLowerGridPoint}{q^{-}}
\newcommand{\mxUpperGridPoint}{q^{+}}
\newcommand{\mxRoundingLogitMatrix}{\matsym{V}} 
\newcommand{\mxRoundingLogitVector}{\vectorsym{v}} 
\newcommand{\mxRoundingLogit}{v} 
\newcommand{\mxRangeRoundOperator}{Q_{\mathrm{RR}}} 
\newcommand{\mxRoundingGate}{h} 
\newcommand{\mxObjective}{\mathcal{L}} 
\newcommand{\mxRoundingRegularizer}{\mathcal{R}_{\mathrm{round}}} 
\newcommand{\mxRegularizationWeight}{\lambda} 
\newcommand{\mxRegularizationExponent}{\beta} 
\newcommand{\mxNormalizedWeight}{z} 
\newcommand{\mxCandidateSet}{\mathcal{C}} 
\newcommand{\mxRangeWidth}{\delta} 
\newcommand{\mxLowCandidateNeighbor}[1]{q_{\mathrm{low}}(#1)} 
\newcommand{\mxHighCandidateNeighbor}[1]{q_{\mathrm{high}}(#1)} 
\newcommand{\mxLowerCandidate}{z^{\downarrow}} 
\newcommand{\mxUpperCandidate}{z^{\uparrow}} 
\newcommand{\mxSigmoidGamma}{\gamma} 
\newcommand{\mxSigmoidZeta}{\zeta} 
\newcommand{\mxCellIndex}{k} 
\newcommand{\mxCellCoordinate}{u} 
\newcommand{\mxLayerWeightCount}{d} 
\newcommand{\mxTeacherInput}{\matsym{X}} 
\newcommand{\mxQuantizedPrefixInput}{\matsym{X}_{\mathrm{q}}} 
\newcommand{\mxStudentInput}{\widetilde{\matsym{X}}} 
\newcommand{\mxLayerWeightTensor}{\matsym{W}} 
\newcommand{\mxTeacherOutput}{\matsym{Y}} 
\newcommand{\mxStudentOutput}{\widetilde{\matsym{Y}}} 
\newcommand{\mxBiasUpdate}{\vectorsym{\delta}_{\mathrm{bias}}}
\newcommand{\mxBiasPenaltyWeight}{\eta_{\mathrm{bias}}} 
\newcommand{\mxScheduleStep}{t} 
\newcommand{\mxTotalSteps}{T} 
\newcommand{\mxWarmupFraction}{w} 
\newcommand{\mxWarmupStep}{t_w} 
\newcommand{\mxScheduleProgress}{p} 
\newcommand{\mxMaxRegularizationWeight}{\lambda_{\max}} 
\newcommand{\mxSearchExponent}{p_{\mathrm{search}}} 
\newcommand{\mxActivationTensor}{\vectorsym{x}} 
\newcommand{\mxCorrectedActivation}{\vectorsym{u}} 
\newcommand{\mxCalibrationDataset}{\mathcal{D}_{\mathrm{cal}}} 
\newcommand{\mxActivationAffineScale}{\vectorsym{\alpha}} 
\newcommand{\mxActivationAffineShift}{\vectorsym{\gamma}} 
\newcommand{\mxFollowingLayerBias}{\vectorsym{b}} 
\newcommand{\mxFoldedLayerWeightTensor}{\matsym{W}'} 
\newcommand{\mxFoldedLayerBias}{\vectorsym{b}'} 

\newcommand{\result}[2]{#1\,\agentsedit{\ensuremath{\pm}\,#2}}
\newcommand{\mxExampleLoss}{\Delta\mathcal{L}_{\mathrm{apx}}} 
\newcommand{\mxExampleCoordinate}{q} 

\newcommand{\mxTaskLoss}{\mathcal{L}_{\mathrm{task}}} 
\newcommand{\mxTaskLossChange}{\Delta\mathcal{L}_{\mathrm{task}}} 
\newcommand{\mxQuantizedWeight}{\vectorsym{w}_{q}} 
\newcommand{\mxWeightError}{\vectorsym{\epsilon}_{\vectorsym{w}}} 
\newcommand{\mxTaskGradient}{\vectorsym{g}_{\vectorsym{w}}} 
\newcommand{\mxTaskHessian}{\matsym{H}_{\vectorsym{w}}} 
\newcommand{\mxAnalysisInput}{\vectorsym{x}} 
\newcommand{\mxAnalysisTarget}{\vectorsym{y}} 
\newcommand{\mxNetworkFunction}{f} 
\newcommand{\mxSampleLoss}{\ell} 

\newcommand{\mxWeightScalar}{w} 
\newcommand{\mxActivationScalar}{x} 
\newcommand{\mxCorrectedActivationScalar}{u} 
\newcommand{\mxAffineScaleScalar}{\alpha} 
\newcommand{\mxAffineShiftScalar}{\gamma} 
\newcommand{\mxCalibrationActivationMatrix}{\matsym{U}} 

\newcommand{\mxBlockAbsMax}{M} 
\newcommand{\mxTensorAbsMax}{M_{\mathrm{tensor}}} 
\newcommand{\mxTensorScale}{g} 
\newcommand{\mxNvBlockScale}{s^{\mathrm{NV}}} 

\newcommand{\mxUpwardScaleProjection}{\Pi^{\uparrow,>}} 
\newcommand{\mxFormatNormalizedInput}{y} 
\newcommand{\mxEffectiveScale}{a} 
\newcommand{\mxScaleSearchCost}{C} 

\newcommand{\mxIntGrid}{\mathcal{Q}_{\mathrm{INT4}}} 
\newcommand{\mxFpGrid}{\mathcal{Q}_{\mathrm{FP4}}} 
\newcommand{\mxBracketSet}{\mathcal{Q}^{(\mathrm{UD})}} 
\newcommand{\mxBracketLower}{a} 
\newcommand{\mxBracketUpper}{b} 
\newcommand{\mxUniformStep}{\Delta_{\mathrm{INT4}}} 
\newcommand{\mxInputCorrelation}{\rho} 
\newcommand{\mxRoundingError}{\boldsymbol{\epsilon}} 
\newcommand{\mxRoundingErrorScalar}{\epsilon} 
\newcommand{\mxRoundingLossChange}{\Delta} 

\newcommand{\mxBracketArgument}{u} 
\newcommand{\mxFigureLossLabel}{\mathcal{L}} 

\newcommand{\mxExampleAssignment}[1]{\vectorsym{w}_{q,#1}} 

\newcommand{\mxMaxGridGap}{g_{\max}} 
\newcommand{\mxGridCount}{m} 
\newcommand{\mxNormalizedWeightVector}{\vectorsym{z}} 
\newcommand{\mxQuantizedNormalizedWeight}{\vectorsym{z}_{q}} 
\newcommand{\mxFixedScaleMatrix}{\matsym{S}} 
\newcommand{\mxMaxRoundingError}{E} 
\newcommand{\mxDescentIndex}{k} 
\newcommand{\mxDescentStep}{s} 
\newcommand{\mxCoordinateBasis}{\vectorsym{e}} 
\newcommand{\mxErrorSign}{\sigma} 

\newcommand{\mxFpCandidateProduct}{\mathcal{M}_{\mathrm{FP4}}} 
\newcommand{\mxWindowFigureLabel}{\mathcal{M}_4} 

\newcommand{\mxNormalizedHessian}{\matsym{H}_{z}} 
\newcommand{\mxScaledSurrogate}{\mathcal{L}_{S}} 

\newcommand*\circled[1]{\tikz[baseline=(char.base)]{%
            \node[shape=circle,fill=GrayHighlight!65,inner sep=1.pt] (char) {\textbf{#1}};}}

\definecolor{wacvblue}{rgb}{0.21,0.49,0.74}
\usepackage[pagebackref,breaklinks,colorlinks,allcolors=wacvblue]{hyperref}

\def\wacvPaperID{*****}
\def\confName{WACV}
\def\confYear{2027}

\title{VisionMX: Unlocking Microscaling Post-Training Quantization\\ for Vision Models}
\author{%
Elad Dror Cohen, Ofir Gordon, Lior Dikstein,
Idan Achituve, Hai Victor Habi\thanks{Corresponding author: Hai Victor Habi, \href{mailto:haivictor.habi@arm.com}{\texttt{haivictor.habi@arm.com}}.}\\
Arm AI Research (AAIR), Arm Ltd %
}

\begin{document}
\maketitle

\begin{abstract}
Microscaling (MX) formats are emerging as a hardware-supported approach to efficient training and inference. They combine low-precision elements with shared block scales, but their impact on vision models remains underexplored.
We systematically investigate post-training MX quantization across vision models and tasks. An analysis of direct conversion identifies three sources of error: block-scale representation, the poor alignment of some small convolutional weight tensors with nonuniform element grids, and the underuse of signed codes by nonnegative activations. These findings motivate \vismix{}, a post-training MX quantization method that optimizes bounded weight rounding and applies a foldable affine correction to activations. We evaluate \vismix{} across image classification, object detection, semantic segmentation, and low-light image enhancement using several MX-style formats. It improves on direct conversion and the evaluated post-training quantization baselines, with the largest performance recoveries in architectures most sensitive to MX conversion.

\end{abstract}

\section{Introduction}
Modern deep learning models have led to significant progress in computer vision, spanning a broad set of tasks and applications such as image classification ~\cite{sandler2018mobilenetv2,cai2023retinexformer}, object detection\cite{ren2016faster,tian2019fcos}, and image restoration\cite{yan2025hvi}.
These advances, however, have come with growing computation and memory demands that complicate efficient deployment and inference.
\agentsreplace{This is especially important for high-resolution inputs and hardware with limited latency, memory bandwidth, or power budgets.}{This is especially important for high-resolution inputs and deployments with tight latency, memory-bandwidth, or power budgets.}

Quantization ~\cite{jacob2018quantization,gholami2022survey,nagel2020adaround,gordon2024eptq,li2021brecq,wei2022qdrop} helps meet these requirements by representing weights and intermediate activations at lower precision, reducing both storage and computation.
Quantization methods are commonly divided into quantization-aware training (QAT) \cite{jacob2018quantization,nagel2022overcoming} and post-training quantization (PTQ) \cite{habi2021hptq,nagel2020adaround,gordon2024eptq,li2021brecq,wei2022qdrop,liu2021post}.
\agentsreplace{QAT retrains a pretrained model, often with labeled data and a task-specific training recipe, so that the model can adapt to quantization during optimization.}{QAT trains or fine-tunes a model while accounting for quantization, often using labeled data and a task-specific training recipe.}
PTQ instead starts from a pretrained model and uses a small calibration set, typically without labels, to choose quantization parameters or apply limited fine-tuning.
\agentsreplace{Vision models quantization}{Quantization of vision models}  ~\cite{nagel2020adaround,li2021brecq,wei2022qdrop,gordon2024eptq,li2023repqvit,zhong2024erq} has been studied extensively in integer formats, including learned rounding, which optimizes weight rounding decisions to preserve layer outputs, as well as block reconstruction, activation-aware reconstruction, and transformer-specific reparameterization methods.

Recently, the Open Compute Project (OCP) introduced Microscaling (MX) formats ~\cite{rouhani2023microscaling}, which combine low-precision element formats with a shared scale over small \agentsreplace{tensor blocks.~\cite{opencompute2023mxalliance,ocp2023mxspec,rouhani2023microscaling}.}{tensor blocks.}
Moreover, NVIDIA introduced NVFP4 \agentsreplace{as it own}{as its own} MX-like format for low-precision computation~\cite{nvidia2025nvfp4}. 
We collectively refer to these MX-style formats as block-scaled formats.
%
%
Recent work on MX quantization for LLMs has reported a significant performance drop when activation and weights are quantized via direct casting to low-bit formats. To mitigate this, several studies have addressed the problem in different ways: (i) \cite{egiazarian2026bridging,gordon2026latmix} propose transforming the feature distribution to better match the MX block structure; and (ii) others explore alternatives to power-of-two scaling factors \cite{fasoli2026finer}. Vision-specific evidence remains limited and mixed, with existing studies reporting substantial MXFP4 degradation in some architectures and settings~\cite{rouhani2023microscaling,chen2025oscillation,wen2026improving}. This leaves a central question: \textbf{When are MX quantization effective for vision inference, and what PTQ adaptations are needed when direct conversion falls short?}

To address this question, we study post-training quantization with MX quantization across diverse vision architectures and tasks. 
We examine how model architecture, element format, block-scale format, and the quantization of weights and activations affect performance. Our analysis reveals substantial variation in sensitivity to MX quantization across architectures. Several evaluated transformer models remain relatively robust, whereas compact convolutional and hybrid models can suffer severe degradation. In addition to variation at the architectural level, the block-scale format can significantly impact accuracy. 
We further note two representation mismatches: (1) a discrepancy between adaptive rounding methods\cite{nagel2020adaround,li2021brecq,wei2022qdrop,gordon2024eptq} and the nonuniform floating-point grid, and (2) signed grids can waste representational capacity when activations are nonnegative due to nonlinearities such as ReLU.

These findings point to two complementary adaptations: more flexible weight rounding on nonuniform element grids and better alignment of nonnegative activations with signed grids. Guided by this analysis, we introduce \vismix{}: a vision-specified MX PTQ approach proposing a new adaptive rounding, the learned quantizer RangeRound, which addresses the mismatch with floating-point values and extending the asymmetric correction~\cite{bhalgat2020lsq+,habi2021hptq} to MX quantization, which we refer to as affine activation correction.
RangeRound addresses the first issue by relaxing the conventional restriction of optimizing the rounding decision only between the two nearest quantization levels. Instead, RangeRound allows each value to select its quantized representation from a fixed local range of quantization levels. This additional flexibility enables the reconstruction objective to better adapt to the nonuniform grid of floating-point formats: values near zero can move across multiple densely spaced quantization levels, while values near the clipping boundary can be rounded either upward or downward to better minimize the reconstruction error.
For activation alignment, we adapt affine reparameterization to the MX setting through Activation Affine Correction (AAC), which shifts and scales activations before quantization and folds the inverse transformation into the subsequent linear operation. In our experiments, we evaluate VisionMX across convolutional, transformer, and hybrid architectures on image classification, object detection, semantic segmentation, and low-light image enhancement. The results show improvements over established PTQ baselines in most evaluated settings under joint 4-bit weight and activation quantization. Moreover, we conduct ablation studies to assess the contribution of each component and examine the effect of the rounding-range.

Our contributions are threefold.
\circled{1} We provide a systematic analysis of post-training MX quantization across diverse vision architectures, characterizing architecture-dependent sensitivity, examining how floating-point formats affect adaptive rounding weight quantization, and identifying the underutilization of signed element grids by nonnegative activation distributions. \circled{2} We introduce \agentsedit{\vismix{}}, a vision-specific MX PTQ method which introduce RangeRound for adaptive weight rounding of float point  with Activation Affine Correction.
    \circled{3} We thoroughly assess the proposed method against established baselines across representative convolutional, transformer, and hybrid vision architectures, observing performance improvements on most models and tasks.

\section{Background and Notation}
Throughout this paper, we adopt the following notation: vectors are denoted by $\vectorsym{x}$ and their scalar components by $x_i$. Matrices and tensors are denoted by $\matsym{A}$.

\subsection{MX Block-Scaled Quantization}
Microscaling (MX) formats combine a low-precision element format with \agentsreplace{a block scales for a block of size}{a shared scale for each block of} $\mxBlockSize$ values~\cite{rouhani2023microscaling}. Let $\mxActivationTensor\in\mathbb{R}^{d}$ a vector of size $d$ we wise to quantized. Assume $\mxBlockSize$ divides $d$, so $\mxNumBlocks=d/\mxBlockSize$ blocks are formed.
The index set for block $\mxBlockIndex$ is
$\mxBlockElementSet_{\mxBlockIndex}
=\{(\mxBlockIndex-1)\mxBlockSize+1,\ldots,
\mxBlockIndex\mxBlockSize\}$ for
$\mxBlockIndex=1,\ldots,\mxNumBlocks$.
Let $\mxFormatSet=\{\mxExampleCoordinate_0,\ldots,\mxExampleCoordinate_{|\mxFormatSet|-1}\}$ denote the finite set of distinct representable values of the unscaled element format, ordered as $\mxExampleCoordinate_0<\cdots<\mxExampleCoordinate_{|\mxFormatSet|-1}$ with $|\mxFormatSet|\ge2$. Let $\mxElementQuantizer$ denote the corresponding element quantizer, $\mxScaleSet$ the legal set of stored scale values, and $\mxScaleQuantizer$ the scale quantizer mapping an ideal nonnegative scale to $\mxScaleSet$.
The quantity $\mxMaxElementMagnitude=\max_{q\in\mxFormatSet}|q|$ is the largest magnitude representable by the element format.
For block $\mxBlockIndex$, the max-based reference scale
$\mxScale_{\mxBlockIndex}^{\star}
=\max_{i\in\mxBlockElementSet_{\mxBlockIndex}}|\mxActivationScalar_i|/\mxMaxElementMagnitude$
is mapped to a legal stored scale:
\begin{equation}
\mxQuantizedScale_{\mxBlockIndex}
=
\mxScaleQuantizer(\mxScale_{\mxBlockIndex}^{\star})
=
\mxScaleQuantizer\!\left(
\frac{
\max_{i\in\mxBlockElementSet_{\mxBlockIndex}}|\mxActivationScalar_i|
}{
\mxMaxElementMagnitude
}
\right)
\in\mxScaleSet.
\end{equation}
\agentsedit{The scale format determines the representable scale values, while the scale-selection rule determines which value is used. For details, see Appendix~\ref{sec:scale_conversion_conventions}.}
The MX-quantized tensor is then
\begin{equation}
\mxQuantizer(\mxActivationTensor)_i
=
\mxQuantizedScale_{\mxBlockIndex}
\mxElementQuantizer\!\left(
\frac{\mxActivationScalar_i}{\mxQuantizedScale_{\mxBlockIndex}}
\right),
\qquad
i\in\mxBlockElementSet_{\mxBlockIndex}.
\end{equation}
%
\ofiredit{Throughout this work, MXFP4 uses four-bit floating-point elements in E2M1 format, comprising one sign bit, two exponent bits, and one mantissa bit. For the shared block scales, we consider two formats: (i) UE8M0, our notation for the unsigned exponent-only 8-bit format defined in the OCP specification~\cite{rouhani2023microscaling}, whose finite representable scales are restricted to powers of two; and (ii) UE5M3, proposed in~\cite{fasoli2026finer}, an unsigned 8-bit floating-point format with five exponent bits and three mantissa bits, enabling finer scale granularity. We denote these configurations by MXFP4-UE8M0 and MXFP4-UE5M3, respectively. We also report results for NVFP4; its scaling conventions are described in Appendix~\ref{sec:scale_conversion_conventions}.}


\section{Related Work}
\textbf{MX Quantization.} The Microscaling (MX) Data Formats, recently introduced by the Open Compute Project (OCP), are driven by key hardware vendors such as \agentsreplace{Arm, AMD, Nvidia, Qualcomm, Intel, Meta, Microsoft and aim}{Arm, AMD, NVIDIA, Qualcomm, Intel, Meta, and Microsoft, and aim} to standardize data formats for neural-network deployment~\cite{rouhani2023microscaling}. MX introduces block-scaled quantization with several numerical formats, including floating-point and integer variants. \agentsreplace{In addition, NVIDIA has developed its own block-scaled format, NVFP4~\cite{egiazarian2026bridging,li2026faar}.}{In addition, NVIDIA developed its own block-scaled format, NVFP4~\cite{nvidia2025nvfp4}. }

Most subsequent MX research has focused on language models. 
MicroScopiQ, MicroMix, MR-GPTQ, and LATMiX explore mixed formats, feature transformations, and reconstruction methods for LLM inference~\cite{ramachandran2025microscopiq,liu2026micromix,egiazarian2026bridging,gordon2026latmix}.  \agentsreplace{Another line of works have focused specifically on NVFP4 format such as FAAR~ \cite{li2026faar} explicitly adapts learned rounding~\cite{nagel2020adaround} to NVFP4's nonuniform grid, using a relaxed lower-or-upper decision for each weight.}{FAAR~\cite{li2026faar} adapts learned rounding~\cite{nagel2020adaround} to NVFP4's nonuniform grid by optimizing a relaxed decision between adjacent representable values for each weight.} Separately, \agentsreplace{\cite{fasoli2026finer} analyze}{a study~\cite{fasoli2026finer} analyzes} how block size interacts with quantized block scales, showing that finer blocks can hurt LLM performance when the scale format has limited dynamic range, and \agentsreplace{propose}{proposes} UE5M3 as an alternative \agentsreplace{shared-scale format}{block-scale format}.


\agentsreplace{Early results on vision application with MX quantization on ImageNet showed}{Early ImageNet results for MX quantization of vision models showed} that post-training MXFP4 can cause substantial accuracy degradation, particularly for compact convolutional networks~\cite{rouhani2023microscaling}. The limited vision-specific literature includes TetraJet~\cite{chen2025oscillation}, which studies oscillation-reduced MXFP4 training for vision transformers rather than PTQ, and ~\cite{wen2026improving} is a recent study that evaluates emulated MX inference for classification, detection, and segmentation but reports substantial MXFP4 accuracy losses.In addition, a recent study~\cite{sen2026characterizing} investigates NVFP4 quantization for image classification, assessing block size, weight precision, and how activation scaling affects performance both with and without retraining.

\textbf{Vision Quantization.} Integer quantization has been extensively studied for vision models through both quantization-aware training and PTQ, including rounding optimization and reconstruction-based methods. We focus on PTQ because it avoids costly model retraining. 
Methods such as AdaRound \cite{nagel2020adaround}, BRECQ \cite{li2021brecq}, \agentsedit{\qdrop{}} \cite{wei2022qdrop}, and EPTQ \cite{gordon2024eptq} have substantially improved low-bit integer PTQ through learned rounding and different optimization strategies, such as blockwise reconstruction, stochastic activation bypass, and network-wide optimization. 
These techniques enable practical 4-bit quantization and improve more challenging \agentsreplace{lower bit settings}{lower-bit settings}, although severe degradation can remain at extremely low precision. 
Earlier CNN-oriented PTQ work established calibration, bias correction, and cross-layer weight equalization as effective mechanisms for reducing quantization error \cite{Nagel_2019_ICCV, NEURIPS2019_c0a62e13}.
More recent work targets vision transformers, whose activation distributions and outliers are particularly difficult to quantize. 
Other methods address additional difficulties arising from attention, Softmax, GELU, LayerNorm, inter-channel variation, and calibration under low-precision quantization \cite{liu2021posttraining,yuan2022ptq4vit,Liu_2023_CVPR,Moon_2024_CVPR,Wu_2025_CVPR}.
RepQ-ViT reparameterizes activation scales, while ERQ reduces errors arising from both weight and activation quantization~\cite{li2023repqvit,zhong2024erq}.
All vision PTQ methods discussed above target integer quantization.

Because MX is a newly emerging data format and PTQ in low-bit regimes is still insufficiently studied with a substantial performance gap, this paper explores post-training MX quantization by analyzing where accuracy degrades across major vision architecture families and proposing \vismix{}, which leverages these findings to combine learned weight rounding with affine activation correction.

\section{Analysis}
\label{sec:analysis}

We study how model architecture, block-scale format, weight structure, and activation distribution affect vision models' sensitivity to MX quantization. The analysis is organized around three questions that characterize the main sources of this sensitivity.
\noindent\fcolorbox{wacvblue}{wacvblue!8}{%
    \parbox{\dimexpr\linewidth-2\fboxsep-2\fboxrule\relax}{%
        \textbf{Q1:} How does sensitivity to MX quantization vary
        across vision architecture families and between weights and
        activations?%
        \\[0.4em]
        \textbf{Q2:} How do nonuniform element grid
affect adaptive rounding algorithm ?
        \\[0.4em]
        \textbf{Q3:} How well do signed element grids represent
activation distributions produced by nonlinear operations?
    }%
}
To identify the sources of degradation, we combine empirical comparisons across vision architectures and theoretical analysis of weight rounding and an examination of activation code utilization. We begin with round-to-nearest evaluations to establish baseline sensitivity before applying learned PTQ adaptations.

\subsection{Architecture Behavior under MX Quantization}
\label{sec:architecture_behavior}

\begin{figure*}[t]
    \centering
    \begin{subfigure}[t]{0.24\textwidth}
        \centering
        \includegraphics[height=1.5in]{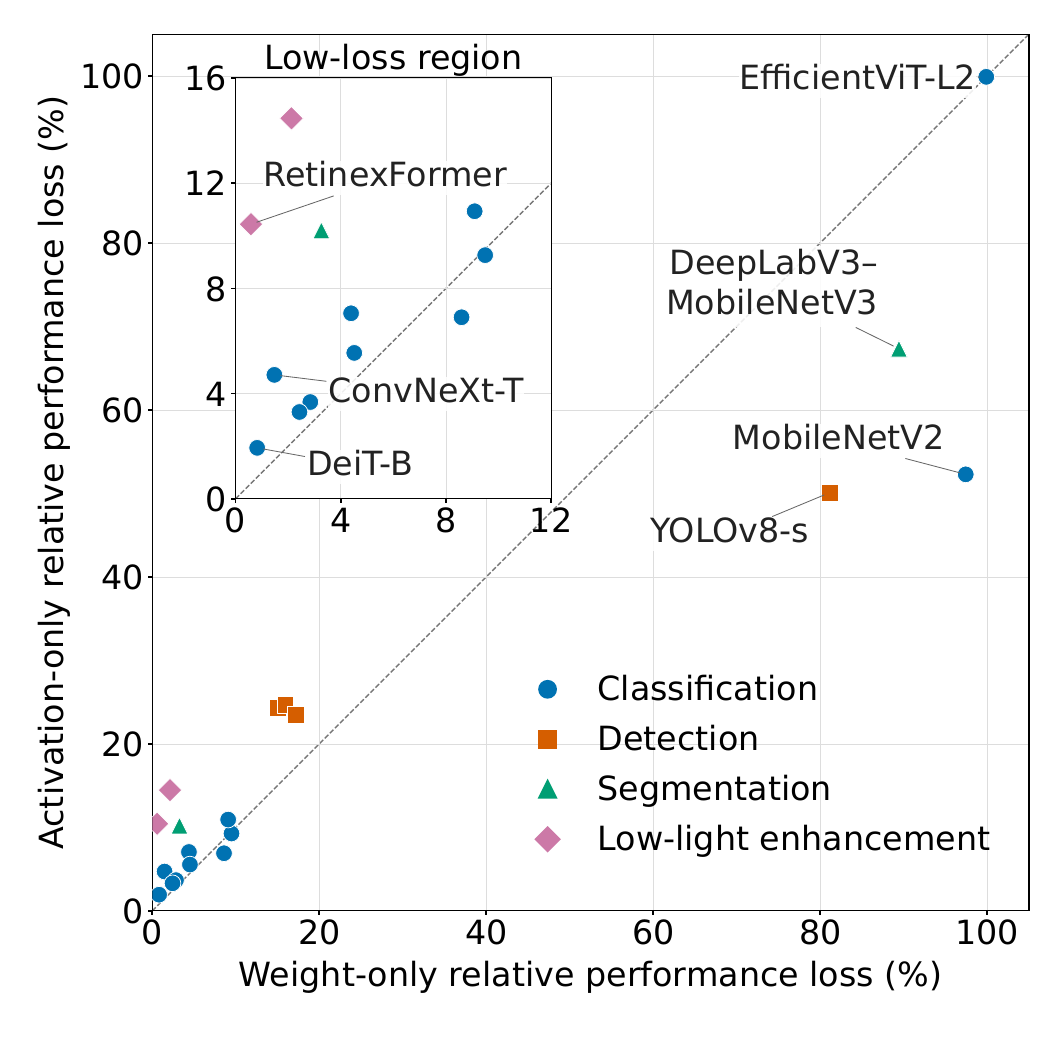}
        \caption{Model sensitivity}
        \label{fig:architecture_family_sensitivity}
    \end{subfigure}\hfill
    \begin{subfigure}[t]{0.30\textwidth}
        \centering
        \includegraphics[height=1.5in]{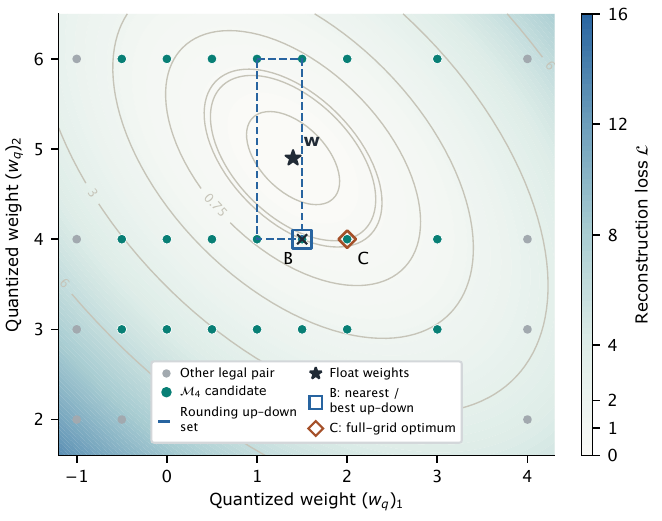}
        \caption{FP4 weight reconstruction}
        \label{fig:fp4_hessian_reconstruction}
    \end{subfigure}\hfill
    \begin{subfigure}[t]{0.43\textwidth}
        \centering
        \includegraphics[height=1.5in]{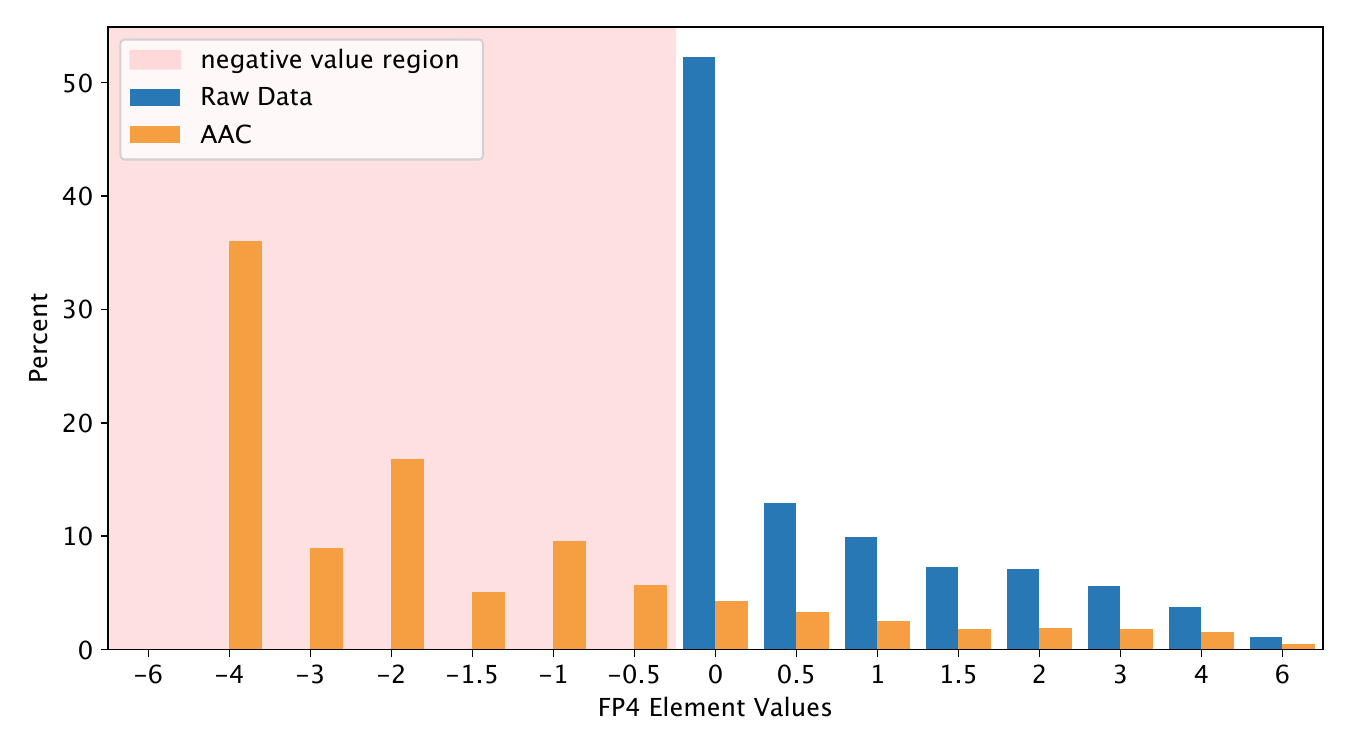}
        \caption{Post-ReLU FP4 occupancy}
        \label{fig:resnet18_relu_fp4_occupancy}
    \end{subfigure}
    \caption{\textbf{Three sources of sensitivity under block-scaled quantization.}
    (a) Relative loss from the floating-point reference under separate weight-only and activation-only MXFP4-\agentsreplace{E8M0}{UE8M0} RTN quantization; each point is a model--task pair, and the inset enlarges the low-loss region.
    (b) Example quadric loss approximation, the contour plot compares the rounding up--down set (dashed rectangle) with the 21-point RangeRound set \(\mxFpCandidateProduct(\mxWeight)\) (teal points). The star denotes the float weights, B the coincident nearest and best assignment in the rounding up--down set, and C the full-grid and RangeRound optimum.
    (c) MXFP4 element-code occupancy  of ReLU output with and without AAC. }
    \label{fig:analysis_three_sensitivities}
\end{figure*}

Here, we compare the impact of applying weight-only and activation-only MXFP4-\agentsreplace{E8M0}{UE8M0} quantization separately, using direct round-to-nearest conversion. Results are presented in Figure~\ref{fig:architecture_family_sensitivity}, with detailed numbers in Appendix~\ref{sec:additional_weight_activation_results}. The results cover classification, object detection, semantic segmentation, and low-light enhancement. For each model--task pair, we report the performance decrease relative to its floating-point reference. This exposes the relative sensitivity of the two quantization modes while retaining each task's evaluation metric.

Sensitivity varies substantially among the evaluated models. DeiT-B, Swin-T, ViT-S/16, and ConvNeXt-Tiny remain within 5\% relative loss in both quantization modes. In contrast, MobileNetV2 loses 97.40\% of its reference accuracy under weight-only quantization and 52.30\% under activation-only quantization. EfficientViT-L2 approaches complete accuracy loss in both settings. These results highlight the pronounced sensitivity of some convolutional models to four-bit MXFP quantization, although the robustness of ConvNeXt-Tiny shows that this behavior is not uniform across convolutional architectures.


At higher precision, MXFP8 RTN preserves performance close to the floating-point reference for most evaluated models, with detailed results reported in Appendix~\ref{sec:additional_mxfp8_results}. 



\subsection{\agentsreplace{Analysis Quantization of Weights}{Analysis of Weight \ofiredit{MX} Quantization}}
\label{sec:convolutional_weight_quantization}
\ofirreplace{Weight quantization has been studied through the lenses of Hessian-based sensitivity analysis~\cite{dong2019hawq,nagel2020adaround,li2021brecq,wei2022qdrop,gordon2024eptq}. To analyze the effect of MX quantization, let \(\mxTaskLoss(\mxWeight)\) denote the expected task loss and define the quantization error as \(\mxWeightError\triangleq\mxWeight-\mxQuantizedWeight\). Assuming a sufficiently small perturbation and a negligible first-order contribution at the pretrained weights, a second-order Taylor expansion gives}
{The effect of weight quantization on model performance has been studied through Hessian-based sensitivity analysis~\cite{dong2019hawq,nagel2020adaround,li2021brecq,wei2022qdrop,gordon2024eptq}. We adopt this perspective to analyze MX weight quantization. Let \(\mxTaskLoss(\mxWeight)\) denote the expected task loss and define the quantization error as \(\mxWeightError\triangleq\mxWeight-\mxQuantizedWeight\). Assuming a sufficiently small perturbation and a negligible first-order contribution at the pretrained weights, a second-order Taylor expansion gives}
\begin{align}
\mxTaskLossChange(\mxQuantizedWeight,\mxWeight)
&\triangleq \mxTaskLoss(\mxQuantizedWeight)-\mxTaskLoss(\mxWeight)\label{eq:weight_loss_quadratic}\\
&\approx \tfrac12\mxWeightError^\top\mxTaskHessian\mxWeightError=\mxExampleLoss(\mxQuantizedWeight;\mxWeight),\nonumber
\end{align}
where \(\mxTaskHessian=\nabla_{\mxWeight}^{2}\mxTaskLoss(\mxWeight)\) is the task-loss Hessian evaluated at the float weights. Appendix~\ref{sec:weight_quantization_derivation} gives the derivation and assumptions.

We next examine whether restricting each weight to its round-down and round-up levels preserves the minimum of the task-loss quadratic approximation \(\mxExampleLoss\) in Eq.~\eqref{eq:weight_loss_quadratic}.

\begin{theorem}[Rounding up--down optimality for INT4 and FP4]
\label{thm:weights_rounding_bracketing}
Assume that \(\mxWeight\in[\min\mxFormatSet,\max\mxFormatSet]^{\mxLayerWeightCount}\) and that \(\mxTaskHessian\) is positive definite and satisfies \((\mxTaskHessian)_{ii}\ge2\sum_{j\ne i}|(\mxTaskHessian)_{ij}|\) for every coordinate \(i\). Define \(\mxBracketSet(\mxWeight)=\prod_{i=1}^{\mxLayerWeightCount}\{\mxBracketLower(\mxWeightScalar_i),\mxBracketUpper(\mxWeightScalar_i)\}\) as the rounding up--down set, where \(\mxBracketLower(\mxBracketArgument)=\max\{\mxExampleCoordinate\in\mxFormatSet:\mxExampleCoordinate\le\mxBracketArgument\}\) and \(\mxBracketUpper(\mxBracketArgument)=\min\{\mxExampleCoordinate\in\mxFormatSet:\mxExampleCoordinate\ge\mxBracketArgument\}\) denote rounding down and up in \(\mxFormatSet\), respectively.

For INT4, let \(\mxIntGrid=\mxUniformStep\{-8,\ldots,7\}\) with common fixed spacing \(\mxUniformStep>0\), and assume that the full-grid minimizer is unique. Then
\begin{align*}
\min\limits_{\mxQuantizedWeight\in\mxIntGrid^{\mxLayerWeightCount}}
\mxExampleLoss(\mxQuantizedWeight;\mxWeight)
=\min\limits_{\mxQuantizedWeight\in\mxBracketSet_{\mathrm{INT4}}(\mxWeight)}
\mxExampleLoss(\mxQuantizedWeight;\mxWeight).
\end{align*}
For unscaled E2M1 FP4,
\begin{align*}
\min_{\mxQuantizedWeight\in\mxFpGrid^{\mxLayerWeightCount}}
\mxExampleLoss(\mxQuantizedWeight;\mxWeight)
\le\min_{\mxQuantizedWeight\in\mxBracketSet_{\mathrm{FP4}}(\mxWeight)}
\mxExampleLoss(\mxQuantizedWeight;\mxWeight).
\end{align*}
\end{theorem}
\noindent\emph{Proof.} See Appendix~\ref{sec:weights_rounding_proof}.

Theorem~\ref{thm:weights_rounding_bracketing} establishes that the unique optimal INT4 assignment lies in the rounding up--down set under the stated conditions. The FP4 inequality follows from inclusion of the rounding up--down candidates in the full codebook. This result concerns the quadratic approximation under the stated task-Hessian conditions; extending it to the exact task loss requires control of the Taylor remainder.

MX quantization uses floating-point formats with a nonuniform element grid, meaning the size of any rounding correction varies depending on where the weight falls within that grid. Figure~\ref{fig:fp4_hessian_reconstruction} illustrates this for FP4 E2M1, whose adjacent spacings are \(0.50\), \(1\), and \(2\) before block scaling. To illustrate their interaction, consider a quadratic task loss centered at the float weights, with fixed block scale \(\mxQuantizedScale=1\) and
\begin{equation}
\mxWeight=\begin{bmatrix}1.40\\4.90\end{bmatrix},
\qquad
\mxTaskHessian=\begin{bmatrix}1&0.50\\0.50&1\end{bmatrix}.
\label{eq:fp4_coupled_reconstruction}
\end{equation}
This illustrative task Hessian satisfies the theorem's condition at equality. The rounding up--down set \(\mxBracketSet(\mxWeight)\) is \(\{1,1.50\}\times\{4,6\}\). Nearest rounding and the best assignment in the rounding up--down set coincide at \(\mxExampleAssignment{B}=(1.50,4)^\top\), with loss \(0.365\). The additional legal assignment \(\mxExampleAssignment{C}=(2,4)^\top\) reduces the loss to \(0.315\), despite increasing squared weight error from \(0.82\) to \(1.17\). The wider candidate set illustrated in Figure~\ref{fig:fp4_hessian_reconstruction} contains this optimum. This motivates the bounded RangeRound search analyzed in Proposition~\ref{prop:rangeround_window}. Appendix~\ref{sec:weights_rounding_analysis} provides the proofs, INT4 comparison, and correlation sweep.
\agentsedit{In addition, Appendix~\ref{sec:block_scale_format_analysis} examines the effect of block-scale format on weight quantization. In weight-only sweeps for ResNet-18 and ConvNeXt-Tiny, UE5M3 achieves the highest ImageNet top-1 accuracy among the evaluated formats. These results are consistent with findings in LLM quantization, where prior work highlights the importance of scale format and recommends UE5M3 scales~\cite{fasoli2026finer}.}

\subsection{\ofiredit{MX Quantization after Nonlinearities}}
\label{sec:nonlinear_activation_quantization}

\ofiredit{Activation quantization introduces a mismatch between activation distributions and the available quantization levels. Nonlinearities such as ReLU produce nonnegative outputs, leaving the negative levels of signed MX element grids unused.} 
\autoref{fig:resnet18_relu_fp4_occupancy} illustrates this restriction for ResNet18's final post-residual ReLU output. The bars report quantized FP4 element-code occupancy; each distribution uses its own dynamically recomputed block scales. The unshifted output is nonnegative and therefore cannot use negative element codes.
For an approximately symmetric signed grid, this leaves roughly half of the nonzero representable values unused, corresponding to an effective loss of approximately one bit of code-space capacity. Changing the shared scale can expand or contract the represented range, but it cannot recenter the distribution or make the unused negative codes accessible.

Applying an affine shift before quantization moves part of the originally nonnegative activation distribution onto the negative side of the signed grid, making previously unused codes accessible. At the same time, the scale adjustment controls its range. 
The resulting change in code utilization diagnoses the mismatch between post-nonlinearity activation distributions and signed block-scaled formats. This observation motivates Activation Affine Correction (AAC), which shifts and scales activations before quantization and folds the inverse transformation into the subsequent linear or convolutional operation.

\section{Method}

Motivated by the distinct challenges identified for weight and activation quantization in Section~\ref{sec:analysis}, our proposed PTQ pipeline, \agentsedit{\vismix{}}, has two complementary components.
First, Activation Affine Correction (AAC) shifts and scales activations before MX activation quantization.
Second, RangeRound optimizes weight-rounding decisions within a bounded region of the nonuniform element grid while keeping the deployed block scales fixed.
\subsection{Activation Affine Correction}
We address the issue in \autoref{fig:resnet18_relu_fp4_occupancy}, where enforcing nonnegativity leaves part of a signed MX quantization grid unused. Following LSQ+ \cite{bhalgat2020lsq+} and HPTQ \cite{habi2021hptq}, we apply an affine transform to better align the feature distribution with the quantization range. Specifically, before dynamic MX activation quantization we shift and scale the input to each linear operation, then fold the inverse transform into the linear operation. The shift recenters activations to better utilize the signed grid, and the scale adjusts the range. 
\ofiredit{We adapt affine correction to the MX block structure by calibrating per-channel shifts and scales block by block, accounting for their shared scale when selecting the correction parameters.}
We call this correction Activation Affine Correction (\agentsreplace{ACC}{AAC}): Let $\mxActivationTensor\in\mathbb{R}^{d}$ be the activation vector input to a linear operation. For each channel $i$, \agentsreplace{ACC}{AAC} applies a positive scale $\mxAffineScaleScalar_i>0$ and shift $\mxAffineShiftScalar_i$.
\begin{equation}
\mxCorrectedActivationScalar_i
=
(\mxActivationScalar_i-\mxAffineShiftScalar_i)
\mxAffineScaleScalar_i,
\qquad i=1,\ldots,d.
\end{equation}
Note that \agentsreplace{Channel parameters}{channel parameters} are shared across spatial positions and calibration samples.
We select the affine parameters to minimize the error obtained after transforming the activations, applying block scaling quantization, and mapping them back to the original domain:
\begin{equation*}
\begin{aligned}
(\mxActivationAffineScale^\star,\mxActivationAffineShift^\star)
=
\operatorname*{arg\,min}_{\substack{
\mxActivationAffineScale\in\mathbb{R}_{>0}^{d}\\
\mxActivationAffineShift\in\mathbb{R}^{d}}}
\sum_{\mxActivationTensor\in\mxCalibrationDataset}
\left\|
\mxActivationTensor
-
\left(
\mxQuantizer(\mxCorrectedActivation)\oslash\mxActivationAffineScale
+\mxActivationAffineShift
\right)
\right\|_2^2,
\end{aligned}
\end{equation*}
where $\oslash$ denotes per-element division.
Implementation details of \agentsreplace{ACC}{AAC} are given in Appendix~\ref{sec:activation_affine_correction_details}.

\subsection{RangeRound for MX Quantization}
\agentsedit{RangeRound extends prior learned-rounding methods~\cite{nagel2020adaround,li2021brecq,wei2022qdrop,gordon2024eptq,li2026faar} by optimizing rounding decisions over a specified range rather than only between adjacent values, addressing the weight-quantization behavior analyzed in Section~\ref{sec:convolutional_weight_quantization}.}


\agentsreplace{Specially, Let}{Specifically, let} $\matsym{W}$ denote the full-precision weight tensor
supplied to RangeRound, and let
$\mxWeight=\operatorname{vec}(\matsym{W})
\in\mathbb{R}^{\mxLayerWeightCount}$\footnote{For clarity, we present the formulation for a fully connected layer with
$\matsym{W}\in\mathbb{R}^{m\times n}$ and
$\mxLayerWeightCount=mn$; the same elementwise formulation applies directly
to convolutional weight tensors without changing their shape.}. Each block $\mxBlockElementSet_{\mxBlockIndex}$ has a fixed representable
scale $\mxQuantizedScale_{\mxBlockIndex}>0$ \agentsedit{selected as described in Appendix~\ref{sec:rangeround_implementation_details}}. For a weight element
$\mxWeightScalar_i$ in block $\mxBlockElementSet_{\mxBlockIndex}$, RangeRound
operates in the normalized element domain:
\begin{equation}
\mxNormalizedWeight_i
=
\operatorname{clip}\!\left(
\frac{\mxWeightScalar_i}{\mxQuantizedScale_{\mxBlockIndex}},
\min \mxFormatSet,\max \mxFormatSet
\right),
\qquad i\in\mxBlockElementSet_{\mxBlockIndex}.
\end{equation}
\agentsreplace{Then, base on normalized element domain the candidate set combines}{Then, in the normalized element domain, the candidate set combines} the representable grid points within
a centered window with the nearest strict lower and upper grid points:
\begin{equation}
\mxCandidateSet_{\delta}(\mxNormalizedWeight)
=
\{q\in\mxFormatSet:
|q-\mxNormalizedWeight|\leq\mxRangeWidth/2\}
\cup
\{
\mxLowCandidateNeighbor{\mxNormalizedWeight},
\mxHighCandidateNeighbor{\mxNormalizedWeight}
\},
\end{equation}
where $\mxLowCandidateNeighbor{\mxNormalizedWeight}
=\max\{q\in\mxFormatSet:q<\mxNormalizedWeight\}$ and
$\mxHighCandidateNeighbor{\mxNormalizedWeight}
=\min\{q\in\mxFormatSet:q>\mxNormalizedWeight\}$
are the nearest strict representable grid points below and above
$\mxNormalizedWeight$, omitting a strict neighbor when it does not exist.

\begin{prop}[Sufficiency of a bounded FP4 candidate window]
\label{prop:rangeround_window}
Assume that \(\mxWeight\in[\min\mxFormatSet,\max\mxFormatSet]^{\mxLayerWeightCount}\) and that the assumptions of Theorem~\ref{thm:weights_rounding_bracketing} hold. Define \(\mxMaxGridGap=\max_{0\le i<|\mxFormatSet|-1}(\mxExampleCoordinate_{i+1}-\mxExampleCoordinate_i)\). Then, for all \(\mxRangeWidth\ge2\mxMaxGridGap\),
\begin{multline*}
\min_{\mxQuantizedWeight\in\mxFormatSet^{\mxLayerWeightCount}}
\mxExampleLoss(\mxQuantizedWeight;\mxWeight)
=\min_{\mxQuantizedWeight\in\prod_{i=1}^{\mxLayerWeightCount}\mxCandidateSet_{\mxRangeWidth}(\mxWeightScalar_i)}
\mxExampleLoss(\mxQuantizedWeight;\mxWeight).
\end{multline*}
For unscaled E2M1 FP4, \(\mxMaxGridGap=2\), so \(\mxRangeWidth=4\) suffices.
\end{prop}
\noindent\emph{Proof.} See Appendix~\ref{sec:rangeround_window_proof}.

\agentsedit{Proposition~\ref{prop:rangeround_window} shows that the unscaled candidate window preserves the minimum of the task-loss quadratic approximation under the stated assumptions. In addition, for MX quantization, Proposition~\ref{prop:rangeround_window} holds under the corresponding assumptions on the normalized weights and the transformed Hessian, as detailed in Appendix~\ref{sec:rangeround_window_proof}.}

\agentsreplace{Now we wise to optimized the rounding of $\mxNormalizedWeight$ to one of the grid point on the set $\mxCandidateSet_{\delta}(\mxNormalizedWeight)$. To achieve we the following soft-assignment, let  $\mxRoundingLogitMatrix\in\mathbb{R}^{m\times n}$ be the RangeRound learns a logit matrix, with vector representation $\mxRoundingLogitVector=\operatorname{vec}(\mxRoundingLogitMatrix)$ and component $\mxRoundingLogit_i$ for each weight, and maps each component through a rectified sigmoid}{We now optimize the rounding of each normalized weight $\mxNormalizedWeight_i$ using its candidate set $\mxCandidateSet_{\delta}(\mxNormalizedWeight_i)$. RangeRound learns a logit matrix $\mxRoundingLogitMatrix\in\mathbb{R}^{m\times n}$, with vectorization $\mxRoundingLogitVector=\operatorname{vec}(\mxRoundingLogitMatrix)$ and one component $\mxRoundingLogit_i$ per weight. 
\ofiredit{Following AdaRound~\cite{nagel2020adaround}, we map each logit through a rectified sigmoid. 
Here, the resulting gate} interpolates between the endpoints of the candidate set:}
\begin{equation}
\begin{aligned}
\mxRoundingGate(\mxRoundingLogit_i)
&=\operatorname{clip}\!\left(
\sigma(\mxRoundingLogit_i)(\mxSigmoidZeta-\mxSigmoidGamma)
+\mxSigmoidGamma,0,1
\right),\\
\widetilde{\mxNormalizedWeight}(\mxRoundingLogit_i;\mxNormalizedWeight_i)
&=\mxLowerCandidate(\mxNormalizedWeight_i)
+\mxRoundingGate(\mxRoundingLogit_i)
\left[
\mxUpperCandidate(\mxNormalizedWeight_i)
-\mxLowerCandidate(\mxNormalizedWeight_i)
\right],
\end{aligned}
\end{equation}
where $\sigma(x)$ denotes an elementwise sigmoid activation function. For a fixed range width $\mxRangeWidth$, we suppress its subscript in $\mxCandidateSet(\mxNormalizedWeight)$ and in the endpoint functions below.
\agentsedit{Here, $\mxLowerCandidate(\mxNormalizedWeight)=\min \mxCandidateSet(\mxNormalizedWeight)$ and $\mxUpperCandidate(\mxNormalizedWeight)=\max \mxCandidateSet(\mxNormalizedWeight)$ are the interval endpoint functions. We use $\mxSigmoidGamma=-0.10$ and $\mxSigmoidZeta=1.10$ as the rectified-sigmoid hyperparameters (Appendix~\ref{sec:rangeround_implementation_details}).} 

\agentsreplace{Resulting in soft deployed weight vector:}{This yields the relaxed weight vector used during optimization:}
\begin{equation}
\begin{aligned}
\squareb{\mxRelaxedWeight}_i
&=
\squareb{\mxRangeRoundOperator(\mxWeight;\mxRoundingLogitVector)}_i=
\mxQuantizedScale_{\mxBlockIndex}
\widetilde{\mxNormalizedWeight}
(\mxRoundingLogit_i;\mxNormalizedWeight_i),
\qquad
i\in\mxBlockElementSet_{\mxBlockIndex}.
\end{aligned}
\end{equation}
To drive the relaxed weights toward legal discrete codes without restricting each weight to a binary lower-or-upper choice, we adapt AdaRound's regularizer~\cite{nagel2020adaround} into a multi-well penalty whose minima coincide with the legal codes.
Using the ordered levels of $\mxFormatSet$, let
$\mxCellIndex_i=
\mxCellIndex(\widetilde{\mxNormalizedWeight}
(\mxRoundingLogit_i;\mxNormalizedWeight_i))$
denote the index of the grid cell containing the relaxed value. The adjacent representable grid points are
$\mxLowerGridPoint_i:=q_{\mxCellIndex_i}$ and
$\mxUpperGridPoint_i:=q_{\mxCellIndex_i+1}$, and the relative coordinate
within the cell is
$\mxCellCoordinate_i:=
(\widetilde{\mxNormalizedWeight}
(\mxRoundingLogit_i;\mxNormalizedWeight_i)-\mxLowerGridPoint_i)/
(\mxUpperGridPoint_i-\mxLowerGridPoint_i)$.
The corresponding regularization penalty for the current weight tensor is
\begin{equation}
\mxRoundingRegularizer
=
\sum_{i=1}^{\mxLayerWeightCount}
\left[
1-\left|2\mxCellCoordinate_i-1\right|^{\mxRegularizationExponent}
\right].
\end{equation}
\begin{figure}[t]
  \centering
  \includesvg[width=\columnwidth,inkscapelatex=false,inkscapearea=page]{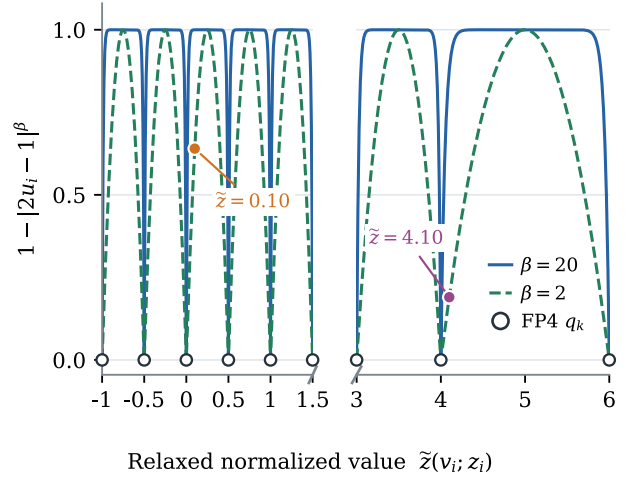}
  \caption{RangeRound per-weight penalty $1-|2\mxCellCoordinate_i-1|^{\mxRegularizationExponent}$ in normalized MX element-code space, for $\mxRegularizationExponent\in\{20,2\}$. The penalty vanishes at legal FP4 codes $q_k$ (open circles) and reaches one at cell midpoints. The axis break omits $(1.5,3)$. Colored markers illustrate relaxed values $\widetilde{\mxNormalizedWeight}=0.10$ and $4.10$ on the $\mxRegularizationExponent=2$ curve.}
  \label{fig:mx_rangeround_regularizer}
\end{figure}
Here, $\mxRegularizationExponent>0$ is the multi-well shape exponent,
annealed as described in
Appendix~\ref{sec:rangeround_implementation_details}. The tensorwise
penalties are summed over all quantized weight tensors in the current
reconstruction unit. As illustrated in
\autoref{fig:mx_rangeround_regularizer}, the penalty vanishes at every
representable grid point and is largest at cell midpoints; exact grid points
are assigned zero penalty with a zero subgradient.

\subsection{Optimization}
\agentsreplace{Here we described the reconstruction which follows \cite{nagel2020adaround,li2021brecq} which as asymmetric: the teacher receives inputs from the floating-point prefix, whereas the student receives inputs from the already-quantized prefix.}{We use asymmetric reconstruction following~\cite{nagel2020adaround,li2021brecq}: the teacher receives inputs from the floating-point prefix, whereas the student receives inputs from the already-quantized prefix.} We propose two variants: Layerwise RangeRound, which integrates \agentsreplace{Range-Round}{RangeRound} into AdaRound-style layer reconstruction \cite{nagel2020adaround}, and Blockwise RangeRound, which applies it in a BRECQ-style block reconstruction \cite{li2021brecq}. We denote them \agentsedit{\vismixLW{}} and \agentsedit{\vismixBW{}}, respectively. \agentsedit{The equations below describe a single affine layer; the blockwise reconstruction units and their output-level objective are defined in Appendix~\ref{sec:block_reconstruction_setup}.}

\agentsedit{Let $\mxTeacherInput$ and $\mxTeacherOutput$ be the floating-point input and output tensors for the current layer, and let $\mxQuantizedPrefixInput$ be the corresponding input from the already-quantized prefix. During reconstruction, we optimize the rounding variables $\mxRoundingLogitMatrix$ and bias update $\mxBiasUpdate$, giving the parameter set $\Omega=\{\mxRoundingLogitMatrix,\mxBiasUpdate\}$.}
\agentsreplace{During optimization the ACC is applied}{During optimization, AAC is applied} before input quantization, and RangeRound relaxes the affine-folded weights to define the soft student path:\begin{table*}[t]
\centering
\color{black}
\def\agentsedit#1{#1}
\def\agentsreplace#1#2{#2}
\renewcommand{\result}[2]{#1\,\ensuremath{\pm}\,#2}
\caption{ImageNet-1K Top-1 classification accuracy (\%). We report the mean and sample standard deviation over five seeds. Within each format and network, darker blue marks the best result and light blue marks the second-best result. }
\label{tab:imagenet1k_classification_other}
\resizebox{\textwidth}{!}{%
\begin{tabular}{llccccccc}
\toprule
Format & Method &
\multicolumn{1}{c}{Tiny} &
\multicolumn{2}{c}{Small} &
\multicolumn{2}{c}{Medium} &
\multicolumn{2}{c}{Large} \\
\cmidrule(lr){3-3}
\cmidrule(lr){4-5}
\cmidrule(lr){6-7}
\cmidrule(lr){8-9}
& &
MobileNetV2 &
ResNet-18 & CaiT-XXS24 &
ConvNeXt-Tiny & EfficientFormer-L3 &
ResNeXt-101 32x8d & EfficientViT-L2 \\
& &
14.30 MB &
46.80 MB & 47.90 MB &
114.40 MB & 125.60 MB &
355.20 MB & 255.00 MB \\
\midrule
Float & -- &
72.15 &
69.76 & 78.39 &
82.52 & 82.55 &
82.83 & 85.06 \\
\midrule

\multirow{5}{*}{\shortstack{MXFP4\\\agentsreplace{E8M0}{UE8M0}}}
& RTN &
0.40 &
57.72 & 69.48 &
76.27 & 63.07 &
64.89 & 0.10 \\

& AdaRound &
\result{41.20}{0.45} &
\result{65.73}{0.10} & \result{65.24}{0.18} &
\result{78.67}{0.05} & \result{79.33}{0.09} &
\result{55.24}{0.46} & \result{2.54}{0.37} \\

& \agentsedit{\qdrop{}} &
\result{57.99}{0.14} &
\result{66.27}{0.13} & \cellcolor[HTML]{DDEBF7}\color{black} \result{73.39}{0.06} &
\result{79.75}{0.06} & \result{80.17}{0.05} &
\result{78.80}{0.06} & \result{4.80}{0.13} \\

& \textbf{\agentsedit{\vismixLW{}} (Ours)} &
\cellcolor[HTML]{DDEBF7}\color{black} \result{64.97}{0.09} &
\cellcolor[HTML]{DDEBF7}\color{black} \result{66.72}{0.08} & \result{70.16}{0.09} &
\cellcolor[HTML]{DDEBF7}\color{black} \result{79.66}{0.03} & \cellcolor[HTML]{DDEBF7}\color{black} \result{80.29}{0.02} &
\cellcolor[HTML]{DDEBF7}\color{black} \result{78.82}{0.08} & \cellcolor[HTML]{DDEBF7}\color{black} \result{78.68}{0.09} \\

& \textbf{\agentsedit{\vismixBW{}} (Ours)} &
\cellcolor[HTML]{9DC3E6}\color{black} \result{65.48}{0.19} &
\cellcolor[HTML]{9DC3E6}\color{black} \result{66.78}{0.06} & \cellcolor[HTML]{9DC3E6}\color{black} \result{73.56}{0.09} &
\cellcolor[HTML]{9DC3E6}\color{black} \result{80.28}{0.06} & \cellcolor[HTML]{9DC3E6}\color{black} \result{80.81}{0.04} &
\cellcolor[HTML]{9DC3E6}\color{black} \result{80.71}{0.06} & \cellcolor[HTML]{9DC3E6}\color{black} \result{81.58}{0.06} \\
\midrule

\multirow{5}{*}{\shortstack{MXFP4\\UE5M3}}
& RTN &
17.35 &
62.28 & 74.57 &
79.37 & 74.46 &
67.41 & 0.12 \\

& AdaRound &
\result{50.96}{0.19} &
\result{67.08}{0.06} & \result{73.06}{0.14} &
\result{80.06}{0.06} & \result{80.53}{0.08} &
\result{75.91}{0.15} & \result{1.35}{0.20} \\

& \agentsedit{\qdrop{}} &
\result{61.35}{0.08} &
\result{67.39}{0.10} & \cellcolor[HTML]{DDEBF7}\color{black} \result{75.50}{0.07} &
\result{80.72}{0.04} & \result{81.00}{0.08} &
\result{80.69}{0.09} & \result{20.18}{0.54} \\

& \textbf{\agentsedit{\vismixLW{}} (Ours)} &
\cellcolor[HTML]{DDEBF7}\color{black} \result{67.07}{0.06} &
\cellcolor[HTML]{DDEBF7}\color{black} \result{68.08}{0.10} & \result{74.57}{0.07} &
\cellcolor[HTML]{DDEBF7}\color{black} \result{81.06}{0.07} & \cellcolor[HTML]{DDEBF7}\color{black} \result{81.36}{0.04} &
\cellcolor[HTML]{DDEBF7}\color{black} \result{81.23}{0.03} & \cellcolor[HTML]{DDEBF7}\color{black} \result{78.44}{0.06} \\

& \textbf{\agentsedit{\vismixBW{}} (Ours)} &
\cellcolor[HTML]{9DC3E6}\color{black} \result{68.27}{0.05} &
\cellcolor[HTML]{9DC3E6}\color{black} \result{68.23}{0.08} & \cellcolor[HTML]{9DC3E6}\color{black} \result{76.18}{0.06} &
\cellcolor[HTML]{9DC3E6}\color{black} \result{81.48}{0.02} & \cellcolor[HTML]{9DC3E6}\color{black} \result{81.68}{0.06} &
\cellcolor[HTML]{9DC3E6}\color{black} \result{81.89}{0.03} & \cellcolor[HTML]{9DC3E6}\color{black} \result{82.15}{0.24} \\
\midrule

\multirow{5}{*}{NVFP4}
& RTN &
13.74 &
64.15 & 75.09 &
80.37 & 77.19 &
77.25 & 0.08 \\

& AdaRound &
\result{54.20}{0.27} &
\result{67.35}{0.07} & \result{73.71}{0.18} &
\result{80.64}{0.07} & \result{80.84}{0.06} &
\result{77.79}{0.06} & \result{1.75}{0.31} \\

& \agentsedit{\qdrop{}} &
\result{62.73}{0.15} &
\result{67.79}{0.08} & \cellcolor[HTML]{DDEBF7}\color{black} \result{75.85}{0.12} &
\result{81.05}{0.06} & \result{81.33}{0.06} &
\result{81.30}{0.05} & \result{14.40}{1.59} \\

& \textbf{\agentsedit{\vismixLW{}} (Ours)} &
\cellcolor[HTML]{DDEBF7}\color{black} \result{67.95}{0.14} &
\cellcolor[HTML]{DDEBF7}\color{black} \result{68.15}{0.07} & \result{74.77}{0.07} &
\cellcolor[HTML]{DDEBF7}\color{black} \result{81.23}{0.05} & \cellcolor[HTML]{DDEBF7}\color{black} \result{81.44}{0.07} &
\cellcolor[HTML]{DDEBF7}\color{black} \result{81.34}{0.03} & \cellcolor[HTML]{DDEBF7}\color{black} \result{81.40}{0.07} \\

& \textbf{\agentsedit{\vismixBW{}} (Ours)} &
\cellcolor[HTML]{9DC3E6}\color{black} \result{68.42}{0.08} &
\cellcolor[HTML]{9DC3E6}\color{black} \result{68.18}{0.08} & \cellcolor[HTML]{9DC3E6}\color{black} \result{76.34}{0.07} &
\cellcolor[HTML]{9DC3E6}\color{black} \result{81.53}{0.03} & \cellcolor[HTML]{9DC3E6}\color{black} \result{81.72}{0.06} &
\cellcolor[HTML]{9DC3E6}\color{black} \result{82.10}{0.03} & \cellcolor[HTML]{9DC3E6}\color{black} \result{82.82}{0.08} \\

\bottomrule
\end{tabular}%
}
\end{table*}

\begin{equation}
\begin{aligned}
\mxStudentInput
&=
\mxQuantizer\!\left((\mxQuantizedPrefixInput-\mxActivationAffineShift)\odot\mxActivationAffineScale\right),\\
\mxStudentOutput
&=
\mxRangeRoundOperator((\mxFoldedLayerWeightTensor);\mxRoundingLogitMatrix)\mxStudentInput
+\mxLayerWeightTensor\mxActivationAffineShift
+\mxBiasUpdate.
\end{aligned}
\end{equation}
\agentsreplace{where}{Here,} $\mxFoldedLayerWeightTensor
=
\mxLayerWeightTensor\operatorname{diag}(\mxActivationAffineScale)^{-1}$ \agentsreplace{is  weights matrix with folded ACC. Then, the reconstruction term is the unnormalized squared Frobenius error,  resulting in objective for the current reconstruction unit is}{is the weight matrix with AAC folded in. The following squared-error objective describes this single affine layer:}
\begin{equation}
\mxObjective(\Omega)
=\|\mxLayerWeightTensor\mxTeacherInput-\mxStudentOutput\|_F^2
+\mxRegularizationWeight\,\mxRoundingRegularizer
+\mxBiasPenaltyWeight\,\|\mxBiasUpdate\|_F^2,
\end{equation}
where $\lambda$ denotes the rounding regularization hyperparameter, and $\mxBiasPenaltyWeight$ is the bias weight-decay regularization. For convolutional layers, the same expressions denote the linear operators induced by the convolution kernels. \ofirreplace{The original layer bias is omitted from both teacher and student reconstruction paths because it appears identically and cancels in the training loss.}{For this single affine layer, the original bias contributes identically to teacher and student outputs and therefore cancels in their squared-error term.} \agentsreplace{In case of Blockwise RangeRound}{For Blockwise RangeRound}, the regularization and bias-penalty terms are summed over the quantized weight tensors and enabled bias updates within the network block.

\agentsedit{The quantization pipeline begins by selecting AAC parameters (Appendix~\ref{sec:aac_initialization}) and weight block scales (Appendix~\ref{sec:scale_conversion_conventions}). Both remain fixed during RangeRound optimization.}  After optimization, each relaxed normalized weight is mapped to a legal element code using the quantizer’s native rounding operation, preserving its format-specific tie behavior. We refer to this conversion as hardening. Each reconstructed layer or network block is persisted immediately after hardening, so subsequent reconstruction units are calibrated against the quantized prefix that will be present at deployment.

\section{Experimental Results}
Here, we evaluate and ablate \vismix{} and demonstrate its benefits across various models and tasks. We begin by evaluating the proposed method \vismix{} across different model families and block-scaled formats on ImageNet classification. We then present object-detection and semantic-segmentation results, with low-light image-enhancement results provided in Appendix~\ref{sec:additional_lowlight_results}. Finally, we ablate our method to show the contribution of each component. 
\subsection{Experimental Setup}
\label{sec:experimental_setup}

\paragraph{Models and tasks.}
For ImageNet-1K classification \cite{russakovsky2015imagenet}, we evaluate models from several architecture
families \cite{he2016deep,xie2017aggregated,sandler2018mobilenetv2,touvron2021going,liu2022convnet,li2022efficientformer,cai2023efficientvit} and capacity ranges, as summarized in
Table~\ref{tab:imagenet1k_classification_other} and Appendix~\ref{sec:additional_vision_transformer_results} (Table~\ref{tab:imagenet1k_classification_vit}). We additionally evaluate object detection with Faster R-CNN, RetinaNet, and FCOS~\cite{ren2016faster,lin2017focal,tian2019fcos}, and semantic segmentation~\cite{chen2017rethinking} on COCO~\cite{coco}, to determine whether the observed behavior extends beyond image classification. Finally, we assess the proposed approach on the low-light image enhancement \cite{cai2023retinexformer,yan2025hvi} task.

\paragraph{Baselines.}
Round-to-nearest (RTN) serves as the direct block-scaled conversion baseline.
We additionally compare against AdaRound~\cite{nagel2020adaround}
and \agentsedit{\qdrop{}}~\cite{wei2022qdrop}. All baseline results
are obtained from our runs using implementations of these methods integrated
with the corresponding block-scaled quantizers. For \agentsedit{\qdrop{}}, we sample the
quantization-bypass mask at the MX block granularity, such that all elements
within a block are either quantized or retained in floating point together
during reconstruction optimization.
We compare these baselines with our complete PTQ pipeline, which combines
RangeRound, AAC, learned bias updates, and the block-granular \agentsedit{\qdrop{}} adaptation which is described in Appendix ~\ref{apx:block_drop} .
Results are reported for both layerwise and blockwise reconstruction. We consider three 4-bit block-scaled formats: MXFP4-UE8M0, MXFP4-UE5M3, and NVFP4. Object-detection results are reported for MXFP4-UE8M0 and MXFP4-UE5M3. 

\paragraph{Calibration and optimization.}
In our experiments, we average results over five random seeds and report the mean and standard deviation across runs. For each task, evaluation is conducted on the corresponding validation set. The calibration set is randomly sampled from the training set, with task-specific sizes (ImageNet-1K: 2,048; COCO: 256; Low-light: 256). Each learned-rounding optimization is run for 10k–20k iterations on a single NVIDIA H100 GPU. Full details are provided in Appendix~\ref{sec:experimental_details}.
\begin{table}[t]
\centering
\color{black}
\def\agentsedit#1{#1}
\def\agentsreplace#1#2{#2}
\renewcommand{\result}[2]{#1\,\ensuremath{\pm}\,#2}
\caption{W4A4 post-training quantization results on COCO val2017 object detection using MXFP4 with UE8M0 and UE5M3 block scales. Higher AP is better. Within each format and network, darker blue marks the best result and light blue marks the second-best result. }
\label{tab:coco-object-detection}
\resizebox{\columnwidth}{!}{%
\begin{tabular}{llccc}
\toprule
Format & Method & Faster R-CNN R50-FPN & RetinaNet R50-FPN
& FCOS R50-FPN \\
\midrule
Float & -- & 36.90 & 36.40 & 39.10 \\
\midrule
\multirow{5}{*}{\shortstack{MXFP4\\UE8M0}}
& RTN & 24.50 & 23.80 & 25.10 \\
& AdaRound &
\result{5.45}{0.18} &
\result{30.83}{0.08} &
\result{31.96}{0.04} \\
& \agentsedit{\qdrop{}} &
\cellcolor[HTML]{DDEBF7}\color{black} \result{32.45}{0.08} &
\result{31.69}{0.10} &
\result{33.90}{0.10} \\
& \textbf{\agentsedit{\vismixLW{}} (Ours)} &
\result{2.11}{0.12} &
\cellcolor[HTML]{9DC3E6}\color{black} \result{32.62}{0.10} &
\cellcolor[HTML]{9DC3E6}\color{black} \result{34.81}{0.03} \\
& \textbf{\agentsedit{\vismixBW{}} (Ours)} & \cellcolor[HTML]{9DC3E6}\color{black} \result{33.15}{0.10} & \cellcolor[HTML]{DDEBF7}\color{black} \result{32.30}{0.08} & \cellcolor[HTML]{DDEBF7}\color{black} \result{34.69}{0.08} \\
\midrule
\multirow{5}{*}{\shortstack{MXFP4\\UE5M3}}
& RTN & 29.00 & 28.00 & 30.60 \\
& AdaRound &
\result{10.64}{0.10} &
\result{33.06}{0.06} &
\result{35.17}{0.06} \\
& \agentsedit{\qdrop{}} &
\cellcolor[HTML]{DDEBF7}\color{black} \result{33.82}{0.08} &
\result{33.13}{0.08} &
\cellcolor[HTML]{DDEBF7}\color{black} \result{35.38}{0.06} \\
& \textbf{\agentsedit{\vismixLW{}} (Ours)} &
\result{4.17}{0.08} &
\cellcolor[HTML]{9DC3E6}\color{black} \result{34.19}{0.04} &
\cellcolor[HTML]{9DC3E6}\color{black} \result{36.42}{0.05} \\
& \textbf{\agentsedit{\vismixBW{}} (Ours)} & \cellcolor[HTML]{9DC3E6}\color{black} \result{34.64}{0.05} & \cellcolor[HTML]{DDEBF7}\color{black} \result{34.07}{0.11} & \cellcolor[HTML]{9DC3E6}\color{black} \result{36.42}{0.06} \\
\bottomrule
\end{tabular}%
}
\end{table}

\subsection{ImageNet Classification Results}
\label{sec:imagenet_classification_results}
Table~\ref{tab:imagenet1k_classification_other} reports ImageNet-1K top-1 accuracy
under joint 4-bit quantization of weights and activations. 
\ofirreplace{Additional vision transformer results appear in Appendix~\ref{sec:additional_vision_transformer_results} (Table~\ref{tab:imagenet1k_classification_vit}).}{Consistent with Section~\ref{sec:architecture_behavior}, the evaluated ViTs are generally less sensitive to MX quantization. We therefore focus the main results on convolutional and hybrid models and report the ViT comparisons in Appendix~\ref{sec:additional_vision_transformer_results}.}
\ofiredit{The results for the CNN and hybrid models} clearly show that \vismix{} provides an accuracy improvement across all models. In some cases, the gain reaches up to 10\%. In addition, the results suggest that larger models are generally more robust to MX quantization, with the exception of EfficientViT-L2. In addition, the results reproduce the architecture-dependent
pattern observed in our analysis shown in Section~\ref{sec:architecture_behavior}. 
The results underscore the benefits of block-scale representation. We find that UE5M3 block scaling consistently improves performance across all models, with gains reaching around 3\% in some cases. In addition, the improvement from using NVFP4 instead of UE5M3 is small in most cases. 
\subsection{Additional Vision Tasks}
\label{sec:additional_task_results}
We next evaluate whether the classification findings extend to object
detection and semantic segmentation. 
Detection results are reported in
Table~\ref{tab:coco-object-detection} using average precision (AP), while segmentation
results are reported in Table~\ref{tab:coco-semantic-segmentation} using mean
intersection-over-union (mIoU).
Additional low-light image-enhancement results are provided in Appendix~\ref{sec:additional_lowlight_results} (Table~\ref{tab:lolv2_real_image_enhancement}).
\begin{table}[ht]
\centering
\color{black}
\def\agentsedit#1{#1}
\def\agentsreplace#1#2{#2}
\renewcommand{\result}[2]{#1\,\ensuremath{\pm}\,#2}
\caption{W4A4 post-training quantization results on COCO val2017 semantic
segmentation. Higher mIoU is better. Within each format and network, darker blue marks the best result and light blue marks the second-best result. }
\label{tab:coco-semantic-segmentation}
\resizebox{\columnwidth}{!}{%
\begin{tabular}{llcc}
\toprule
Format & Method & DeepLabV3 R50 & DeepLabV3 MobileNetV3 \\
\midrule
Float & -- & 66.20 & 60.00 \\
\midrule
\multirow{5}{*}{\shortstack{MXFP4\\UE8M0}}
& RTN & 56.90 & 4.80 \\
& AdaRound &
\result{59.72}{0.20} &
\result{20.08}{0.61} \\
& \agentsedit{\qdrop{}} &
\cellcolor[HTML]{DDEBF7}\color{black} \result{62.61}{0.19} &
\result{47.21}{0.46} \\
& \textbf{\agentsedit{\vismixLW{}} (Ours)} &
\result{62.24}{0.09} &
\cellcolor[HTML]{DDEBF7}\color{black} \result{52.96}{0.21} \\
& \textbf{\agentsedit{\vismixBW{}} (Ours)} & \cellcolor[HTML]{9DC3E6}\color{black} \result{63.13}{0.08} & \cellcolor[HTML]{9DC3E6}\color{black} \result{53.31}{0.28} \\
\midrule
\multirow{5}{*}{\shortstack{MXFP4\\UE5M3}}
& RTN & 58.90 & 11.60 \\
& AdaRound &
\result{62.36}{0.18} &
\result{39.39}{0.50} \\
& \agentsedit{\qdrop{}} &
\result{63.81}{0.13} &
\result{52.29}{0.23} \\
& \textbf{\agentsedit{\vismixLW{}} (Ours)} &
\cellcolor[HTML]{DDEBF7}\color{black} \result{64.41}{0.09} &
\cellcolor[HTML]{DDEBF7}\color{black} \result{56.02}{0.21} \\
& \textbf{\agentsedit{\vismixBW{}} (Ours)} & \cellcolor[HTML]{9DC3E6}\color{black} \result{64.60}{0.03} & \cellcolor[HTML]{9DC3E6}\color{black} \result{56.93}{0.11} \\
\midrule
\multirow{5}{*}{NVFP4}
& RTN & 61.20 & 18.30 \\
& AdaRound &
\result{62.80}{0.09} &
\result{38.66}{0.11} \\
& \agentsedit{\qdrop{}} &
\cellcolor[HTML]{9DC3E6}\color{black} \result{64.43}{0.11} &
\result{50.07}{0.54} \\
& \textbf{\agentsedit{\vismixLW{}} (Ours)} &
\cellcolor[HTML]{DDEBF7}\color{black} \result{64.35}{0.07} &
\cellcolor[HTML]{DDEBF7}\color{black} \result{56.46}{0.10} \\
& \textbf{\agentsedit{\vismixBW{}} (Ours)} & \result{64.24}{0.06} & \cellcolor[HTML]{9DC3E6}\color{black} \result{56.56}{0.12} \\
\bottomrule
\end{tabular}%
}
\end{table}
Across detection and segmentation, direct block-scaled conversion produces \ofiredit{substantial performance degradation}. Existing PTQ baselines \ofiredit{provide varying degrees of recovery}, whereas \ofiredit{\vismixBW{} outperforms them in most} of the evaluated model--format combinations. \ofiredit{\vismixBW{}} improves detection by \ofiredit{5.64--9.59} AP and segmentation by \ofiredit{3.04--48.51} mIoU relative to RTN, \ofiredit{supporting the effectiveness of the combined adaptations for dense prediction}.

\subsection{Ablation Studies}
\label{sec:ablation_studies}

\agentsedit{We ablate RangeRound (RR), the learned bias update, and Activation Affine Correction (AAC) for MXFP4 elements with UE8M0 block scales. Table~\ref{tab:bw-fixed-column-ablation} reports blockwise ImageNet top-1 results for MobileNetV2 and CaiT-XXS24 with \qdrop{}. The separate layerwise ablation without \qdrop{} is reported in Appendix~\ref{sec:layerwise_ablation} (Table~\ref{tab:lw-fixed-column-ablation}); its optimization settings differ from those of the blockwise ablation.}
\begin{table}[ht]
\centering
\caption{VisionMX-BW ablation with MXFP4-UE8M0. ImageNet Top-1 accuracy is reported in percent.}
\label{tab:bw-fixed-column-ablation}

\begin{tabular}{@{}ccc rr@{}}
\toprule
RR & Bias & AAC & MobileNetV2 & CaiT-XXS24 \\
\midrule
\ding{55} & \ding{55} & \ding{55} & 52.73 & 73.15 \\
\ding{55} & \ding{55} & \ding{51} & 64.54 & 73.27 \\
\ding{55} & \ding{51} & \ding{55} & 54.91 & 73.16 \\
\ding{55} & \ding{51} & \ding{51} & 64.84 & 73.55 \\
\ding{51} & \ding{55} & \ding{55} & 56.12 & 73.35 \\
\ding{51} & \ding{55} & \ding{51} & 65.24 & 73.50 \\
\ding{51} & \ding{51} & \ding{55} & 57.42 & 73.30 \\
\ding{51} & \ding{51} & \ding{51} & 65.53 & 73.56 \\
\bottomrule
\end{tabular}

\end{table}

\agentsedit{With RR and AAC enabled, adding the learned bias update raises top-1 accuracy from 65.24\% to 65.53\% on MobileNetV2 and from 73.50\% to 73.56\% on CaiT-XXS24, gains of 0.29 and 0.06 percentage points. With RR and bias enabled, AAC raises accuracy from 57.42\% to 65.53\% and from 73.30\% to 73.56\%, gains of 8.11 and 0.26 points. With bias and AAC enabled, adding RR raises accuracy from 64.84\% to 65.53\% and from 73.55\% to 73.56\%, gains of 0.69 and 0.01 points. The configuration with all three components has the highest top-1 accuracy among the evaluated blockwise rows for both models.}

\section{Conclusion}

This work studies post-training Microscaling quantization across convolutional, transformer, and hybrid vision architectures. Our analysis shows that, among the evaluated models, compact convolutional and hybrid architectures tend to be more sensitive to MX quantization, while transformers are generally more robust. We identify two sources of accuracy degradation: (i) nonuniform floating-point grids, which can limit weight reconstruction when rounding is restricted to adjacent levels; and (ii) nonnegative activations produced by nonlinearities, which leave negative quantization codes unused.

We address these limitations by introducing a new quantization method, \vismix{}, that extends adaptive rounding to nonuniform floating-point grids through range-rounding, allowing weight assignments to be optimized over a bounded range of quantization levels. Additionally, we adapt affine activation correction to the MX block structure, shifting and scaling activations to utilize the signed grid better.

Our experimental results show that the proposed method outperforms established PTQ baselines on MXFP4-\agentsreplace{E8M0}{UE8M0}, MXFP4-UE5M3, and NVFP4, across a broad range of vision tasks. Finally, using \vismix{}, we have shown that MX-aware post-training quantization enables state-of-the-art performance in most benchmarked models. Moreover, these results demonstrate that the new MX data format is well suited for vision applications.




{
    \small
    \bibliographystyle{ieeenat_fullname}
    \bibliography{references}
}

\clearpage
\appendix
\section{Implementation Details}

\subsection{\agentsedit{Experimental Details}}
\label{sec:experimental_details}

Unless otherwise stated, we use the following settings in all experiments. We quantize the input activations and weights of all linear operations, including the input and output layers. Weight block scales are initialized by minimizing the mean squared error between the original and quantized weights, as described in Appendix~\ref{sec:scale_conversion_conventions}. We use the Adam optimizer with weight decay $0.01$ and a warmup fraction of $0.2$. When \qdrop{} is enabled, its probability is $0.5$.

\paragraph{ImageNet-1K classification.}
\agentsedit{We select calibration images with a deterministic class-round-robin sampler and use a batch size of 32. We run reconstruction for 20,000 steps per layer or block, except for ViT-S/16, for which we use 10,000 steps. For layerwise reconstruction, we set the bias learning rate to $10^{-3}$ and disable both the bias penalty and \qdrop{}. For blockwise reconstruction, we set the bias learning rate to $10^{-5}$ and the bias penalty to $10^{-3}$, and enable \qdrop{}. For MobileNetV2 with layerwise UE8M0 quantization, we set the rounding learning rate to $0.003$ and the regularization weight to $1.0$. We evaluate all 50,000 validation images using the preprocessing associated with each pretrained checkpoint and report top-1 accuracy.}

\paragraph{COCO detection and segmentation.}
\agentsedit{We select calibration images from train2017 with a class-aware round-robin sampler. We run adaptive reconstruction for 20,000 steps per layer, with a rounding learning rate of $10^{-3}$ and a regularization weight of $0.01$. We use a batch size of 1 to evaluate all 5,000 val2017 images. For detection, we report bounding-box average precision (AP). For segmentation, we resize predictions to the target mask resolution and report mean intersection-over-union (mIoU) over classes with nonzero union.}

\paragraph{Low-light image enhancement.}
\agentsedit{We use deterministic calibration crops from the LOLv2-Real Train/Low split. We run both layerwise and blockwise reconstruction for 15,000 steps and set the rounding learning rate and regularization weight to $10^{-2}$. We enable \qdrop{} only for blockwise reconstruction. We evaluate at native resolution with a batch size of 1 and report per-image RGB PSNR, SSIM, and LPIPS-Alex without mean correction.}

\subsection{\agentsedit{Quantization Format Implementation Details}}
\label{sec:scale_conversion_conventions}

\agentsedit{
\paragraph{MXFP4.}
We use signed E2M1 elements with maximum magnitude 6,
blocks of 32 values, and no tensor-level scale. Define
\[
\mxBlockAbsMax_{\mxBlockIndex}
=\max_{i\in\mxBlockElementSet_{\mxBlockIndex}}
 |\mxActivationScalar_i|,
\qquad
\mxScale_{\mxBlockIndex}^{\star}
=\mxBlockAbsMax_{\mxBlockIndex}/6.
\]
The positive scale codebooks are
\begin{align*}
{\mxScaleSet}^{\mathrm{UE8M0}}
&=\{2^e:e=-127,\ldots,127\},\\
{\mxScaleSet}^{\mathrm{UE5M3}}
&=\{(1+k/8)2^e:k=0,\ldots,7,\ e=-14,\ldots,16\}.
\end{align*}
For a scale codebook $\mxScaleSet$, define strict upward
projection with endpoint clamping:
\[
\Pi_{\mxScaleSet}^{\uparrow,>}(\mxScale)
=\min\bigl(\{u\in\mxScaleSet:u>\mxScale\}
          \cup\{\max\mxScaleSet\}\bigr).
\]
Our implementation selects
\begin{align*}
\mxQuantizedScale_{\mxBlockIndex}^{\mathrm{UE8M0}}
&=\Pi_{{\mxScaleSet}^{\mathrm{UE8M0}}}^{\uparrow,>}
  (\mxScale_{\mxBlockIndex}^{\star}),\\
\mxQuantizedScale_{\mxBlockIndex}^{\mathrm{UE5M3}}
&=\Pi_{{\mxScaleSet}^{\mathrm{UE5M3}}}^{\uparrow,>}
  (\mxScale_{\mxBlockIndex}^{\star}),
\end{align*}
Exact grid values advance to the next scale,
unless already at the maximum. Zero blocks receive the
minimum scale: $2^{-127}$ for UE8M0 and $2^{-14}$ for UE5M3.
Values exceeding the scaled E2M1 range are saturated.

\paragraph{NVFP4.}
NVFP4 uses E2M1 elements with E4M3 scales shared by groups of 16~\cite{egiazarian2026bridging} and an FP32 tensor-level scale~\cite{nvidia2025nvfp4}. With $\mxTensorAbsMax=\max_i|\mxActivationScalar_i|$, the multiplicative reconstruction factor is
\[
\mxTensorScale=
\begin{cases}
1, & \mxTensorAbsMax=0,\\
\max\{\mxTensorAbsMax/(448\cdot6),\,2^{-126}\},
& \mxTensorAbsMax>0.
\end{cases}
\]
The tensor maximum and $\mxTensorScale$ are computed in FP32; normalization casts the factor to the input dtype. Writing the resulting normalized tensor as $\vectorsym{\mxFormatNormalizedInput}=\mxActivationTensor/\mxTensorScale$, max-based local scale selection is
\[
\mxNvBlockScale_{\mxBlockIndex}
=\mxUpwardScaleProjection_{{\mxScaleSet}^{\mathrm{E4M3}}}
\left(
\frac{\max_{i\in\mxBlockElementSet_{\mxBlockIndex}}
|\mxFormatNormalizedInput_i|}{6}
\right).
\]
The positive E4M3 codebook contains 126 scales, including subnormals, from $2^{-9}$ to $448$. Projection preserves exact grid values and clamps at these endpoints; zero blocks receive $2^{-9}$. The effective scale and reconstruction are
\begin{align*}
\mxEffectiveScale_{\mxBlockIndex}
&=\mxTensorScale\,\mxNvBlockScale_{\mxBlockIndex},\\
\widehat{\mxActivationScalar}_i
&=\mxEffectiveScale_{\mxBlockIndex}\,
\mxElementQuantizer\!\left(
\mxActivationScalar_i/\mxEffectiveScale_{\mxBlockIndex}
\right).
\end{align*}
The quantizer returns the effective scale after folding in the global factor. We use the NVFP4 format with upward-rounded local scales.
}

\subsubsection{\agentsedit{Element rounding and scale selection}}
E2M1 elements are rounded to the nearest representable value, with ties resolved to an even encoding and saturation at $\pm6$. For weights, we select the block scale by minimizing reconstruction MSE ($\mxSearchExponent=2$) over all positive values in the corresponding scale codebook:
\[
\mxScaleSearchCost_{\mxBlockIndex}(\mxScale)
=
\frac{1}{\mxBlockSize}
\sum_{i\in\mxBlockElementSet_{\mxBlockIndex}}
\left|
\mxFormatNormalizedInput_i
-\mxScale\,\mxElementQuantizer(
\mxFormatNormalizedInput_i/\mxScale)
\right|^2,
\qquad \mxScale\in\mxScaleSet.
\]
For MXFP4, $\vectorsym{\mxFormatNormalizedInput}=\mxActivationTensor$. For NVFP4, $\vectorsym{\mxFormatNormalizedInput}=\mxActivationTensor/\mxTensorScale$, with the tensor-level factor $\mxTensorScale$ held fixed during the local scale search. Equal costs select the smaller scale, and zero blocks receive the minimum positive scale.

\subsection{RangeRound Implementation Details}
\label{sec:rangeround_implementation_details}
RangeRound obtains initialized block scales from the active quantizer before
optimizing rounding logits. The selected scales
and zero-points remain fixed throughout RangeRound optimization.

The rectified sigmoid in the main text uses
$\mxSigmoidGamma=-0.10$ and $\mxSigmoidZeta=1.10$.
For zero-based step $\mxScheduleStep$, total steps $\mxTotalSteps$,
and warmup fraction $\mxWarmupFraction$, the implementation sets
\begin{equation}
\begin{aligned}
\mxWarmupStep
&=\lfloor\mxWarmupFraction\mxTotalSteps\rfloor,\\
\mxScheduleProgress(\mxScheduleStep)
&=
\operatorname{clip}\!\left(
\frac{\mxScheduleStep-\mxWarmupStep}
{\mxTotalSteps-1-\mxWarmupStep},0,1
\right),\\
\mxRegularizationWeight(\mxScheduleStep)
&=
\mxMaxRegularizationWeight\,
\mxScheduleProgress(\mxScheduleStep),\\
\mxRegularizationExponent(\mxScheduleStep)
&=
20+(2-20)\mxScheduleProgress(\mxScheduleStep).
\end{aligned}
\end{equation}

One seeded random cached calibration batch is drawn at each step.
Experimental hyperparameters and task-specific overrides are
provided in Appendix~\ref{sec:experimental_details}.

\subsection{\agentsedit{Blockwise Reconstruction Setup}}
\label{sec:block_reconstruction_setup}

\agentsedit{A blockwise reconstruction unit $\mathcal{U}$ is a complete executable module, rather than a sequence represented by one weight matrix. Specifically, for a calibration minibatch, let $\matsym{X}_{\mathrm{q},\mathcal{U}}$ be the input captured at the boundary of $\mathcal{U}$ after the already-quantized network prefix, and let $\mxTeacherOutput_{\mathcal{U}}$ be the structured output of the corresponding floating-point reference unit. The student output is $\mxStudentOutput_{\mathcal{U}}=f_{\mathcal{U}}^{\mathrm{Q}}(\matsym{X}_{\mathrm{q},\mathcal{U}};\Omega_{\mathcal{U}})$, where $\Omega_{\mathcal{U}}$ contains the relaxed rounding variables and enabled bias updates of the quantized layers inside $\mathcal{U}$. Both full block forward passes include their intervening operations, residual paths, nonlinearities, normalization operations, and original internal biases. For dynamic reference roots that require replay, the floating-point root is evaluated on the captured student boundary input instead.}

\agentsedit{Let $\mathcal{T}_{\mathcal{U}}$ index the selected floating-point tensors at the block output. The block-output reconstruction error and bias penalty are}
\begin{equation}
\begin{aligned}
\mathcal{E}_{\mathcal{U}}
&=\frac{1}{N_{\mathcal{U}}}
\sum_{j\in\mathcal{T}_{\mathcal{U}}}
\|\mxStudentOutput_{\mathcal{U},j}
-\mxTeacherOutput_{\mathcal{U},j}\|_F^2,\\
\mathcal{P}_{\mathcal{U}}
&=
\sum_{\substack{\ell\in\mathcal{U}:\\
\vectorsym{\delta}_{\mathrm{bias},\ell}\ \mathrm{enabled}}}
\frac{
\|\vectorsym{\delta}_{\mathrm{bias},\ell}\|_F^2/C_\ell
}{
\|\vectorsym{y}^{\mathrm{ref}}_\ell\|_F^2/N_\ell+\epsilon
}.
\end{aligned}
\end{equation}
\agentsedit{where $N_{\mathcal{U}}=\sum_{j\in\mathcal{T}_{\mathcal{U}}}\operatorname{numel}(\mxTeacherOutput_{\mathcal{U},j})$ is the number of scalar elements in the selected block output tensors. Here $C_\ell$ is the number of bias-update elements, $N_\ell$ is the number of reference-output elements, and $\epsilon=10^{-8}$ by default.}
\agentsedit{At optimization step $\mxScheduleStep$, the objective is}
\begin{equation}
\begin{aligned}
\mathcal{L}_{\mathcal{U}}(\Omega_{\mathcal{U}})
&=\mathcal{E}_{\mathcal{U}}
+\mxRegularizationWeight(\mxScheduleStep)
\sum_{\ell\in\mathcal{U}}\mathcal{R}_{\mathrm{round},\ell}\\
&\quad+\mxBiasPenaltyWeight\mathcal{P}_{\mathcal{U}}.
\end{aligned}
\end{equation}
\agentsedit{Here $\mathcal{R}_{\mathrm{round},\ell}$ is the rounding penalty defined in the main text for quantized weight tensor $\ell$, and $\mxRegularizationWeight(\mxScheduleStep)$ follows the warmup schedule in Appendix~\ref{sec:rangeround_implementation_details}.}

\subsection{Activation Affine Correction Details}
\label{sec:activation_affine_correction_details}
{
    \let\subsection\paragraph
    Activation affine correction stores a per-channel scale
$\mxActivationAffineScale$ and shift $\mxActivationAffineShift$
with each quantized layer. The transform is applied before
dynamic block activation quantization:
\[
\mxCorrectedActivation
=
\operatorname{diag}(\mxActivationAffineScale)
(\mxActivationTensor-\mxActivationAffineShift).
\]
Before weight quantization, the inverse transform is folded
into the following floating-point linear operation:
\begin{equation}
\mxFoldedLayerWeightTensor
=
\mxLayerWeightTensor
\operatorname{diag}(\mxActivationAffineScale)^{-1},
\qquad
\mxFoldedLayerBias
=
\mxFollowingLayerBias
+\mxLayerWeightTensor\mxActivationAffineShift.
\end{equation}
Replacing activation quantization by the identity therefore
preserves the original computation algebraically, and up to
floating-point rounding in the implementation.

For convolutions, weight division and bias compensation are
performed per input channel within each convolution group.
The bias compensation sums the original kernel multiplied
by the corresponding channel shift over input channels and
kernel positions. This also covers depthwise convolution.
At inference, spatial padding is applied before the affine
transform and activation quantization, and the convolution
then uses zero internal padding. Calibration instead uses
the captured real layer-input positions, excluding spatial
padding.

\subsection{Initialization of Activation Affine Correction}
\label{sec:aac_initialization}

The standard grid-search initializer calibrates supported
layers sequentially using captured inputs with activation
quantization disabled. It directly searches the stored
scale and shift. Write
$\alpha_j=(\mxActivationAffineScale)_j$ and
$\mxAffineShiftScalar_j=(\mxActivationAffineShift)_j$.
Let $x_{p,j}$ denote calibration activation $j$ at sample
position $p$, where $p=1,\ldots,P$ indexes all collected
batch and spatial or token positions, excluding convolution
spatial padding.
For samples $x_{p,j}$ and real channels $R_{\mxBlockIndex}$
in block $\mxBlockIndex$, define
\begin{align*}
u_{p,j}
&=\alpha_j(x_{p,j}-\mxAffineShiftScalar_j),\\
\widehat{x}_{p,j}
&=\frac{(\mxQuantizer(\mxCalibrationActivationMatrix))_{p,j}}
        {\alpha_j}
  +\mxAffineShiftScalar_j,\\
E_{\mxBlockIndex}
&=\frac{1}{P|R_{\mxBlockIndex}|}
  \sum_{p=1}^{P}\sum_{j\in R_{\mxBlockIndex}}
  (\widehat{x}_{p,j}-x_{p,j})^2,
\end{align*}
where $\mxCalibrationActivationMatrix=[u_{p,j}]$ and
$\mxQuantizer$ includes quantization and dequantization.
Quantizer-required channel padding is excluded from the
reconstruction error.

For MX formats without a tensor-level scale, each channel
block is optimized independently. Starting from
$\alpha_j=1$ and $\mxAffineShiftScalar_j=0$, the initializer
first searches a shared scale and shift for the whole block.
It alternates shift and scale searches, stopping when a
complete sweep fails to reduce $E_{\mxBlockIndex}$ or the
configured sweep limit is reached. With channel-granularity
correction, the resulting shared parameters initialize a
subsequent per-channel search.

The shift-search interval for channel $j$ is
\[
\left[
\min\{0,\min_p x_{p,j}\},
\max\{0,\max_p x_{p,j}\}
\right].
\]
The shared-block search uses the corresponding extrema
over all real channels in the block. The initial shift grid
is uniformly spaced over this interval. The initial scale
grid is
\[
\left\{2^{r/G_\alpha}:r=0,\ldots,G_\alpha-1\right\}.
\]
Scale refinement is performed in $\log_2\alpha$, within
$0\leq\log_2\alpha<1$. Each grid includes the incumbent
parameter by appending it and removing duplicates;
the incumbent does not replace a grid point.
After each grid evaluation, the search interval is narrowed
to one preceding grid spacing on either side of the best
parameter, clipped to the search domain.

Each channel sweep starts with $\mathcal H=\emptyset$ and
visits every real channel once. Before each visit, priorities
are recomputed and the unvisited channel with the largest
reconstruction MSE is chosen:
\[
i^\star
=
\arg\max_{i\in R_{\mxBlockIndex}\setminus\mathcal H}
\frac{1}{P}\sum_{p=1}^{P}
(\widehat{x}_{p,i}-x_{p,i})^2.
\]
Ties select the lowest channel index. The selected channel's
shift and scale are searched in that order against the
whole-block objective $E_{\mxBlockIndex}$, after which the
channel is added to $\mathcal H$. Accepted updates do not
reset $\mathcal H$ within a sweep. An optional channel-priority
metric uses
\[
\frac{\sum_p(\widehat{x}_{p,i}-x_{p,i})^2}
     {\sum_p x_{p,i}^2+\epsilon_{\mathrm{FP32}}},
\]
where $\epsilon_{\mathrm{FP32}}$ is FP32 machine epsilon;
candidate acceptance still uses $E_{\mxBlockIndex}$.

For this independent-block search, a candidate is accepted
only when its reconstruction MSE is strictly smaller than
the incumbent's. The final block transform is retained only
if it improves on identity. The default configuration uses
$G_\gamma=17$ shift-grid points, $G_\alpha=16$ scale-grid points,
two refinement rounds per search, at most eight shared-block
sweeps, and one channel sweep with one shift--scale
alternation per channel. These settings are configurable.

For NVFP4, the dynamic tensor-level scale couples channel
blocks. Its grid-search path therefore scores candidate
updates using full-tensor reconstruction error, preserving
the original calibration-batch boundaries. Acceptance uses
a relative FP32 numerical tolerance rather than the
independent-block strict comparison. Optional subset-certified
and backpropagation-based calibration paths use separate
configurations.

The additive notation remains equivalent through
$\beta_j=-\alpha_j\mxAffineShiftScalar_j$, giving
$u_{p,j}=\alpha_jx_{p,j}+\beta_j$; however, the implemented
grid search operates directly on
$(\alpha_j,\mxAffineShiftScalar_j)$.
}

\section{Analysis}

\subsection{Second-Order Analysis of Weight Quantization}
\label{sec:weight_quantization_derivation}
We follow the Taylor-expansion and Hessian-chain-rule structure of EPTQ~\cite{gordon2024eptq}, Appendices B.2 and C.1, while retaining weight interactions. Define
\[
\mxTaskLoss(\mxWeight)
=\mathbb{E}_{\mxAnalysisInput,\mxAnalysisTarget}
[\mxSampleLoss(\mxNetworkFunction(\mxAnalysisInput;\mxWeight),\mxAnalysisTarget)].
\]
The expectation is over the task's input--target distribution; for analysis of a single weight tensor, all other weights are held fixed. Suppose this expected loss is twice continuously differentiable near the float weights, and differentiation can be interchanged with expectation. With \(\mxWeightError=\mxWeight-\mxQuantizedWeight\) and \(\mxTaskGradient=\nabla_{\mxWeight}\mxTaskLoss(\mxWeight)\), Taylor expansion yields
\begin{equation}
\begin{aligned}
\mxTaskLossChange
={}&-\mxTaskGradient^\top\mxWeightError
+\tfrac12\mxWeightError^\top\mxTaskHessian\mxWeightError+o(\|\mxWeightError\|_2^2).
\end{aligned}
\label{eq:weight_loss_taylor_full}
\end{equation}
At a stationary float solution, \(\mxTaskGradient=0\); more generally, the linear contribution must be negligible on the perturbations considered to obtain Eq.~\eqref{eq:weight_loss_quadratic}. Pretraining alone does not ensure exact stationarity. Large quantization errors or changes across nonsmooth activation boundaries can also invalidate this local approximation.

\subsection{Block Scale Format Analysis}
\label{sec:block_scale_format_analysis}
\agentsreplace{This appendix reports a staged diagnostic for MXFP4 block-scale format selection across ResNet-18, ConvNeXt-Tiny, and DeiT-Small.}{This appendix reports a staged diagnostic for MXFP4 block-scale format selection across ResNet-18 and ConvNeXt-Tiny.}
Each run evaluates ImageNet top-1 accuracy with group size 32, round-to-nearest quantization.
The sweep changes only the unsigned exponent scale format while holding the remaining quantization setup fixed within each model and mode.
\agentsreplace{The results show that format selection materially changes accuracy in these evaluated settings, but they do not identify a universal best scale format.}{Among the evaluated block-scale formats, UE5 achieves the highest weight-only accuracy for both ResNet-18 and ConvNeXt-Tiny. The preferred format differs in the activation-only settings.}
\begin{figure*}[t]
    \centering
    \begin{subfigure}[t]{0.48\textwidth}
        \centering
        \includegraphics[width=\linewidth]{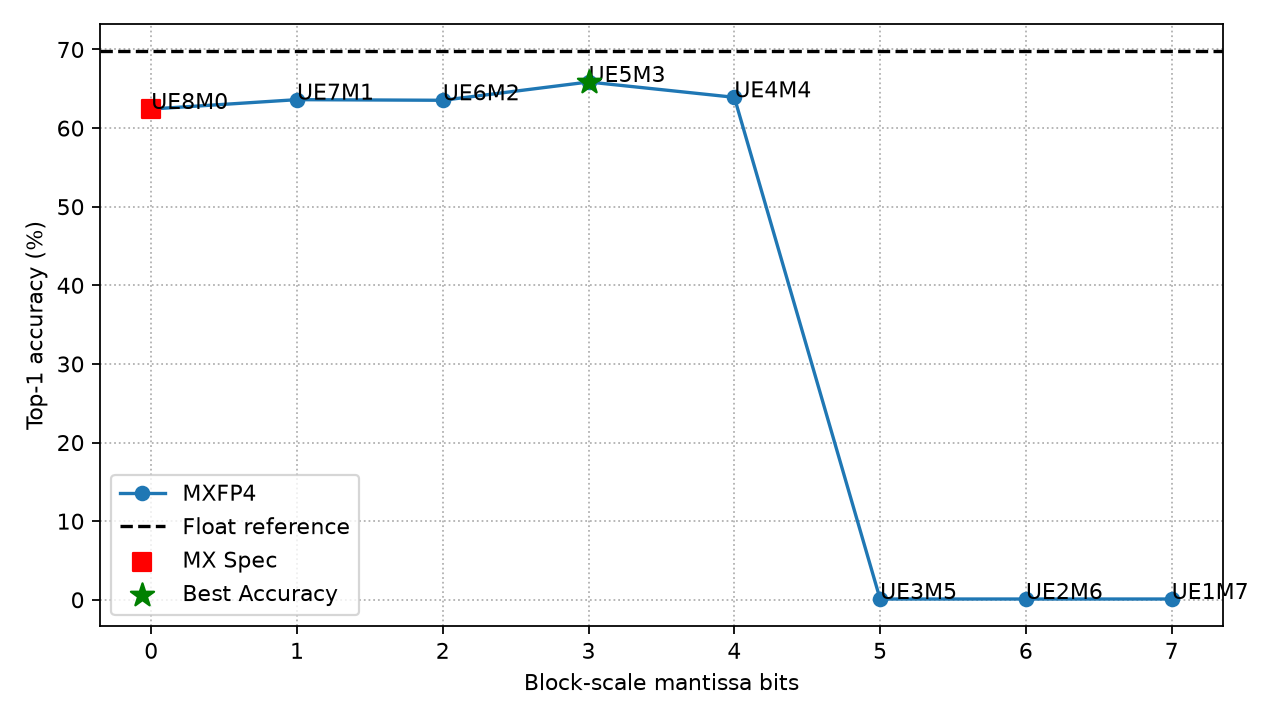}
        \caption{Weight-only}
    \end{subfigure}
    \hfill
    \begin{subfigure}[t]{0.48\textwidth}
        \centering
        \includegraphics[width=\linewidth]{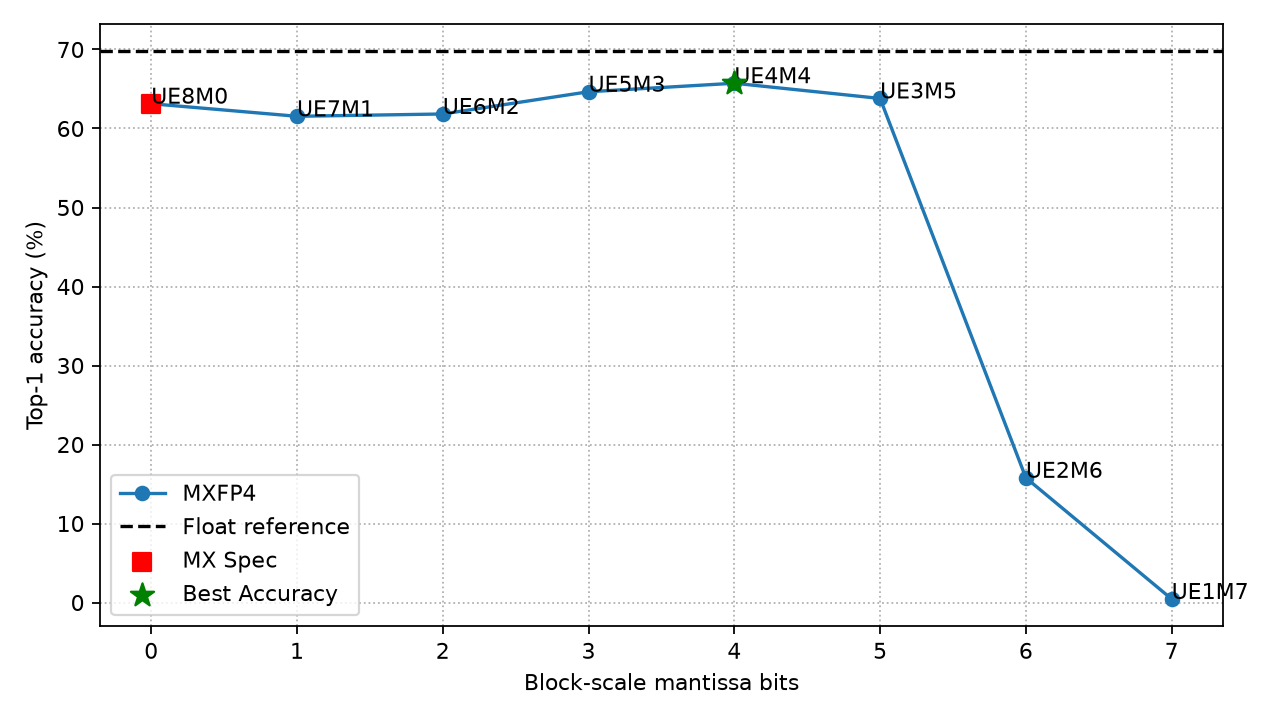}
        \caption{Activation-only}
        \label{fig:resnet18_mxfp4_block_scale_activation_accuracy}
    \end{subfigure}
    \caption{ResNet-18 MXFP4 block-scale diagnostic on ImageNet top-1 accuracy. Both runs use group size 32, RTN, no bias correction, and no AdaRound. In the weight-only setting, activations remain in floating point; in the activation-only setting, weights remain in floating point. The plots show UE8M0 through UE1M7.}
    \label{fig:resnet18_mxfp4_block_scale_accuracy}
\end{figure*}

\begin{figure*}[t]
    \centering
    \begin{subfigure}[t]{0.48\textwidth}
        \centering
        \includegraphics[width=\linewidth]{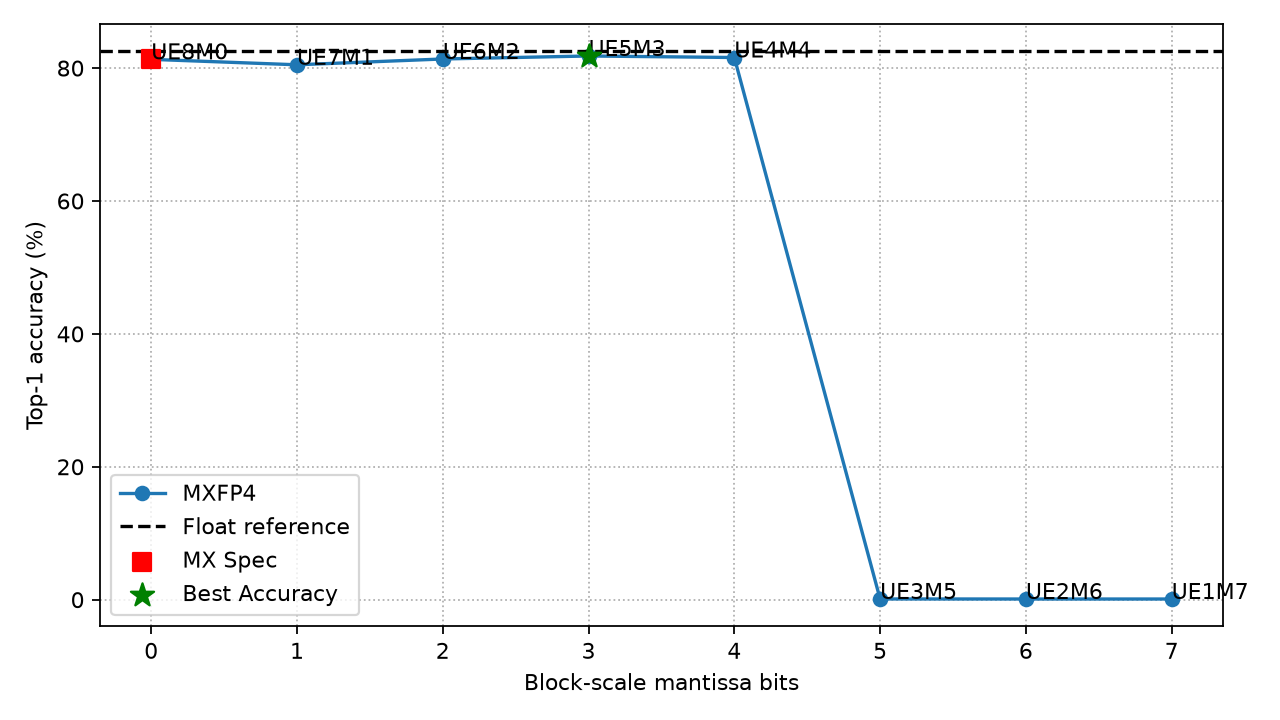}
        \caption{Weight-only}
        \label{fig:convnext_tiny_mxfp4_block_scale_weight_accuracy}
    \end{subfigure}
    \hfill
    \begin{subfigure}[t]{0.48\textwidth}
        \centering
        \includegraphics[width=\linewidth]{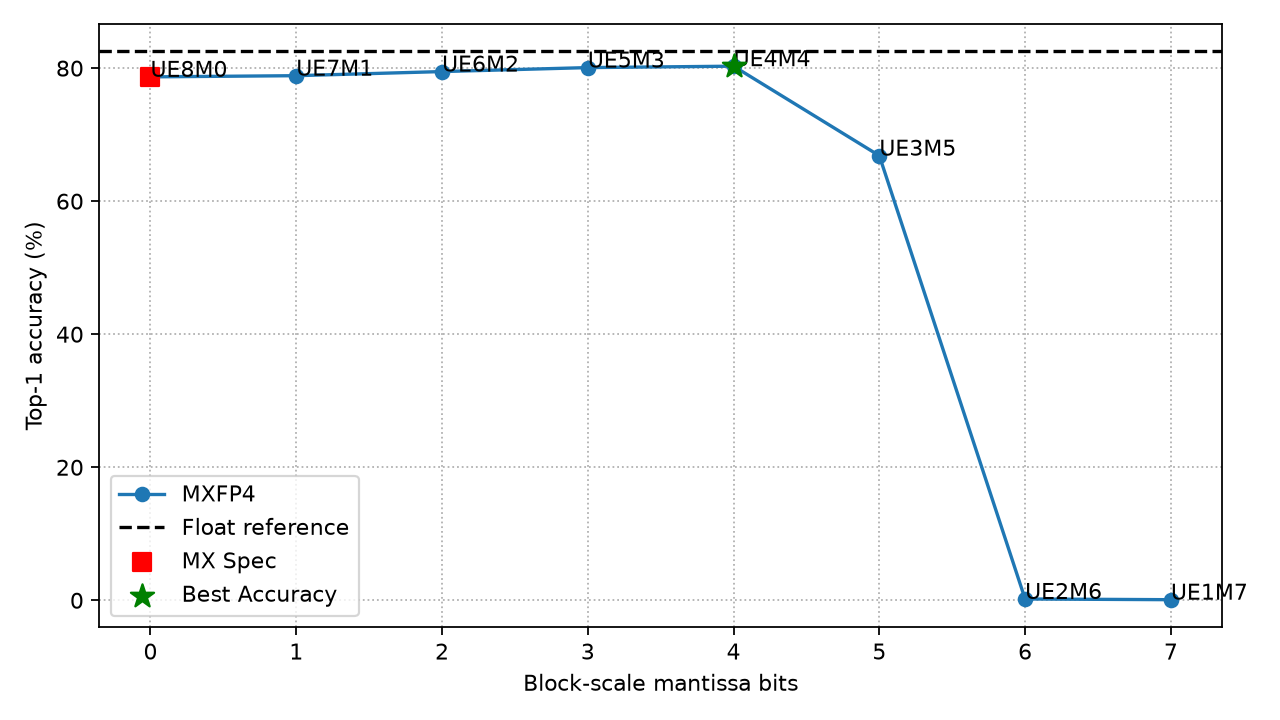}
        \caption{Activation-only}
        \label{fig:convnext_tiny_mxfp4_block_scale_activation_accuracy}
    \end{subfigure}
    \caption{ConvNeXt-Tiny MXFP4 block-scale diagnostic on ImageNet top-1 accuracy. Both runs use group size 32, RTN, no bias correction, and no AdaRound. In the weight-only setting, activations remain in floating point; in the activation-only setting, weights remain in floating point. The plots show UE8M0 through UE1M7.}
\end{figure*}

\begin{table}[t]
    \centering
    \caption{ImageNet top-1 accuracy for the block-scale diagnostic.     W and A denote weight-only and activation-only settings. Best results are selected among the plotted formats UE8M0 through UE1M7; the ResNet-18 activation-only UE4M4 accuracy is estimated from the plotted point. UE8 is the MX-spec scale format. \agentsedit{UE5 denotes UE5M3, with five exponent bits and three mantissa bits.}}
    \label{tab:mxfp4_block_scale_summary}

    \begin{tabular}{@{}lcrrc@{}}
        \toprule
        Model & Mode & Float & UE8 & Best \\
        \midrule
        ResNet-18 & W & 69.76 & 62.43 & 65.87 (UE5M3) \\
        ResNet-18 & A & 69.76 & 63.14 & $65.8$ (UE4M4) \\
        ConvNeXt-Tiny & W & 82.51 & 81.32 & 81.82 (UE5M3) \\
        ConvNeXt-Tiny & A & 82.51 & 78.69 & 80.29 (UE4M4) \\
        \bottomrule
    \end{tabular}%
\end{table}

\begin{table}[t]
    \centering
    \caption{\agentsedit{Unsigned eight-bit block-scale formats in the diagnostic sweep. The ranges list positive scale values before any optional global multiplier.}}
    \label{tab:mxfp4_block_scale_codebooks}
    \begin{tabular}{@{}lccc@{}}
        \toprule
        Format & E/M bits & Bias & Positive scale range \\
        \midrule
        UE8M0 & 8/0 & 0 & $2^{-127}$--$2^{127}$ \\
        UE7M1 & 7/1 & 63 & $2^{-62}$--$1.5\cdot2^{64}$ \\
        UE6M2 & 6/2 & 31 & $2^{-30}$--$1.75\cdot2^{32}$ \\
        UE5M3 & 5/3 & 15 & $2^{-14}$--$1.875\cdot2^{16}$ \\
        UE4M4 & 4/4 & 7 & $2^{-6}$--$1.9375\cdot2^{8}$ \\
        UE3M5 & 3/5 & 3 & $2^{-2}$--$1.96875\cdot2^{4}$ \\
        UE2M6 & 2/6 & 1 & $1$--$1.984375\cdot2^{2}$ \\
        UE1M7 & 1/7 & 0 & $2$--$3.984375$ \\
        \bottomrule
    \end{tabular}%
\end{table}
\agentsedit{In Table~\ref{tab:mxfp4_block_scale_codebooks}, UE8M0 uses exponent-only powers of two with floor rounding. UE7M1 through UE1M7 use an exponent bias of $2^{E-1}-1$ and round the mantissa to the nearest code; their exponent and mantissa are clamped at the listed endpoints, with no subnormal scales emitted. UE0M8 rounds to the nearest multiple of $1/255$ after clipping at one. Values that round to the UE0M8 zero code are replaced by the minimum normal FP32 value, $2^{-126}$. The diagnostic uses a group size of 32 and selects a max-based, no-clipping scale before applying each scale format. The format labels UE8 through UE1 in the plots abbreviate UE8M0 through UE1M7.}

Among the plotted formats UE8M0 through UE1M7, ResNet-18 shows several scale formats above UE8M0, with the highest weight-only result at UE5M3 and the highest activation-only result at UE4M4.
ConvNeXt-Tiny is less sensitive in weight-only mode but still improves from UE8 to UE5, while its activation-only sweep improves most at UE4.



\section{\agentsreplace{Weights Rounding Analysis}{Weight Rounding Analysis}}
\label{sec:weights_rounding_analysis}
\subsection{Proof of Theorem~\ref{thm:weights_rounding_bracketing}}
\label{sec:weights_rounding_proof}
Here, we prove Theorem~\ref{thm:weights_rounding_bracketing}. We begin by explicitly stating the diagonal-dominance assumption on the Hessian, which appears in the machine-learning optimization literature~\cite{wang2021infinite,liu2019preconditioning,rame2023rewarded,tang2024sinkhorn}.
\begin{assumption}[Diagonal-dominance]
The task-loss Hessian at the float weights satisfies
\begin{equation*}
\mxTaskHessian=((\mxTaskHessian)_{ij})_{i,j=1}^{\mxLayerWeightCount}\succ0,
\qquad
(\mxTaskHessian)_{ii}\ge2\sum_{j\ne i}|(\mxTaskHessian)_{ij}|.
\end{equation*}
\end{assumption}
 In this work, we strengthen the usual row-diagonal-dominance condition by a factor of two to derive our rounding results.
\begin{proof}
Let
\[
\mxQuantizedWeight^\star=
\arg\min_{\mxQuantizedWeight\in\mxIntGrid^{\mxLayerWeightCount}}
\mxExampleLoss(\mxQuantizedWeight;\mxWeight)
\]
be the unique full-grid minimizer. Suppose, for contradiction, that
\(\mxQuantizedWeight^\star\notin\mxBracketSet_{\mathrm{INT4}}(\mxWeight)\).
Define \(\mxRoundingError=\mxWeight-\mxQuantizedWeight^\star\) and
\(\mxMaxRoundingError=\|\mxRoundingError\|_\infty\), and choose
\(\mxDescentIndex\) such that
\(|\mxRoundingErrorScalar_{\mxDescentIndex}|=\mxMaxRoundingError\).
Since a coordinate outside its uniform-grid rounding up--down set has
error magnitude at least \(\mxUniformStep\), we have
\(\mxMaxRoundingError\ge\mxUniformStep\).

Let \(\mxErrorSign=\operatorname{sign}(\mxRoundingErrorScalar_{\mxDescentIndex})\)
and, for \(\mxDescentStep>0\), set
\(\mxQuantizedWeight'=\mxQuantizedWeight^\star+
\mxErrorSign\mxDescentStep\mxCoordinateBasis_{\mxDescentIndex}\),
where \(\mxCoordinateBasis_{\mxDescentIndex}\) is the
\(\mxDescentIndex\)-th coordinate basis vector.
For any feasible such assignment, define
\[
\mxRoundingLossChange=
\mxExampleLoss(\mxQuantizedWeight';\mxWeight)
-\mxExampleLoss(\mxQuantizedWeight^\star;\mxWeight).
\]
Expanding the quadratic gives
\begin{align}
\mxRoundingLossChange
&=-\mxErrorSign\mxDescentStep(\mxTaskHessian\mxRoundingError)_{\mxDescentIndex}
+\tfrac12(\mxTaskHessian)_{\mxDescentIndex\mxDescentIndex}\mxDescentStep^2\nonumber\\
&\overset{\mathrm{(a)}}{\le}
-\mxDescentStep\mxMaxRoundingError
\left((\mxTaskHessian)_{\mxDescentIndex\mxDescentIndex}
-\sum_{j\ne\mxDescentIndex}|(\mxTaskHessian)_{\mxDescentIndex j}|\right)
+\tfrac12(\mxTaskHessian)_{\mxDescentIndex\mxDescentIndex}\mxDescentStep^2\nonumber\\
&\overset{\mathrm{(b)}}{\le}
\tfrac12(\mxTaskHessian)_{\mxDescentIndex\mxDescentIndex}
\mxDescentStep(\mxDescentStep-\mxMaxRoundingError).
\label{eq:rounding_coordinate_descent}
\end{align}
Here, (a) uses
\(\mxErrorSign\mxRoundingErrorScalar_{\mxDescentIndex}=\mxMaxRoundingError\)
and \(|\mxRoundingErrorScalar_j|\le\mxMaxRoundingError\);
(b) uses factor-two diagonal dominance.

Take \(\mxDescentStep=\mxUniformStep\). Since
\(\mxMaxRoundingError\ge\mxUniformStep\), the updated coordinate lies
between \((\mxQuantizedWeight^\star)_{\mxDescentIndex}\) and
\(\mxWeightScalar_{\mxDescentIndex}\), within the representable interval.
Thus \(\mxQuantizedWeight'\) is a feasible grid assignment distinct from
\(\mxQuantizedWeight^\star\). Equation~\eqref{eq:rounding_coordinate_descent}
gives a nonpositive loss change. A strict decrease contradicts optimality,
while equality contradicts uniqueness. Hence
\(\mxQuantizedWeight^\star\in\mxBracketSet_{\mathrm{INT4}}(\mxWeight)\), so
\[
\min_{\mxQuantizedWeight\in\mxIntGrid^{\mxLayerWeightCount}}
\mxExampleLoss(\mxQuantizedWeight;\mxWeight)
=
\min_{\mxQuantizedWeight\in\mxBracketSet_{\mathrm{INT4}}(\mxWeight)}
\mxExampleLoss(\mxQuantizedWeight;\mxWeight).
\]

For FP4, the inclusion
\(\mxBracketSet_{\mathrm{FP4}}(\mxWeight)\subseteq
\mxFpGrid^{\mxLayerWeightCount}\) directly gives
\[
\min_{\mxQuantizedWeight\in\mxFpGrid^{\mxLayerWeightCount}}
\mxExampleLoss(\mxQuantizedWeight;\mxWeight)
\le
\min_{\mxQuantizedWeight\in\mxBracketSet_{\mathrm{FP4}}(\mxWeight)}
\mxExampleLoss(\mxQuantizedWeight;\mxWeight).
\]
\end{proof}
\subsection{Proof of Proposition~\ref{prop:rangeround_window}}
\label{sec:rangeround_window_proof}
\begin{proof}
Let
\[
\mxQuantizedWeight^\star\in
\arg\min_{\mxQuantizedWeight\in\mxFormatSet^{\mxLayerWeightCount}}
\mxExampleLoss(\mxQuantizedWeight;\mxWeight),
\]
and suppose, for contradiction, that
\(\mxQuantizedWeight^\star\notin
\prod_{i=1}^{\mxLayerWeightCount}\mxCandidateSet_{\mxRangeWidth}(\mxWeightScalar_i)\).
Set \(\mxRoundingError=\mxWeight-\mxQuantizedWeight^\star\) and
\(\mxMaxRoundingError=\|\mxRoundingError\|_\infty\), and choose
\(\mxDescentIndex\) with
\(|\mxRoundingErrorScalar_{\mxDescentIndex}|=\mxMaxRoundingError\).
Since the candidate window contains every legal code within the closed
radius \(\mxRangeWidth/2\), an excluded coordinate has error magnitude
greater than \(\mxRangeWidth/2\). Thus
\[
\mxMaxRoundingError>\mxRangeWidth/2\ge\mxMaxGridGap.
\]

Let \(\mxErrorSign=\operatorname{sign}(\mxRoundingErrorScalar_{\mxDescentIndex})\),
and let \(\mxDescentStep\) be the gap to the next adjacent code toward
the in-range weight \(\mxWeightScalar_{\mxDescentIndex}\).
This code exists because the selected coordinate has nonzero error and
the reference weight lies within the representable interval. Therefore
\[
\mxQuantizedWeight'=\mxQuantizedWeight^\star+
\mxErrorSign\mxDescentStep\mxCoordinateBasis_{\mxDescentIndex}
\in\mxFormatSet^{\mxLayerWeightCount},
\qquad
0<\mxDescentStep\le\mxMaxGridGap<\mxMaxRoundingError.
\]
Applying Eq.~\eqref{eq:rounding_coordinate_descent} gives
\[
\begin{aligned}
&\mxExampleLoss(\mxQuantizedWeight';\mxWeight)
-\mxExampleLoss(\mxQuantizedWeight^\star;\mxWeight)\\
&\quad\le\tfrac12(\mxTaskHessian)_{\mxDescentIndex\mxDescentIndex}
\mxDescentStep(\mxDescentStep-\mxMaxRoundingError)<0,
\end{aligned}
\]
contradicting optimality. Hence every full-grid minimizer belongs to the
candidate product. Since that product is a subset of the full grid,
\begin{multline*}
\min_{\mxQuantizedWeight\in\mxFormatSet^{\mxLayerWeightCount}}
\mxExampleLoss(\mxQuantizedWeight;\mxWeight)
=
\min_{\mxQuantizedWeight\in
\prod_{i=1}^{\mxLayerWeightCount}\mxCandidateSet_{\mxRangeWidth}(\mxWeightScalar_i)}
\mxExampleLoss(\mxQuantizedWeight;\mxWeight).
\end{multline*}
No uniqueness assumption is required. For unscaled E2M1 FP4,
\(\mxMaxGridGap=2\), so every \(\mxRangeWidth\ge4\) suffices.
\end{proof}
\agentsedit{
\begin{corollary}[Candidate-window sufficiency for fixed-scale MX quantization]
\label{cor:mx_fixed_scale_window}
Let \(\mxFixedScaleMatrix\) be diagonal, with
\((\mxFixedScaleMatrix)_{ii}=\mxQuantizedScale_{\mxBlockIndex}>0\)
for each coordinate \(i\in\mxBlockElementSet_{\mxBlockIndex}\), where
\(\mxQuantizedScale_{\mxBlockIndex}\) is the fixed representable scale
shared by that MX block. Let
\(\mxFormatSet=\{\mxExampleCoordinate_0<\cdots<
\mxExampleCoordinate_{\mxGridCount-1}\}\) be the element grid, and define
\[
\begin{aligned}
\mxNormalizedWeightVector&=\mxFixedScaleMatrix^{-1}\mxWeight,
&
\mxNormalizedHessian&=\mxFixedScaleMatrix^\top\mxTaskHessian\mxFixedScaleMatrix,\\
\mxMaxGridGap&=\max_{0\le j<\mxGridCount-1}
(\mxExampleCoordinate_{j+1}-\mxExampleCoordinate_j).
\end{aligned}
\]
Assume that
\(\mxNormalizedWeightVector\in[\min\mxFormatSet,\max\mxFormatSet]^{\mxLayerWeightCount}\),
so normalization requires no clipping, and that
\[
\mxNormalizedHessian\succ0,\qquad
(\mxNormalizedHessian)_{ii}\ge
2\sum_{j\ne i}|(\mxNormalizedHessian)_{ij}|
\quad\text{for every }i.
\]
For normalized codes
\(\mxQuantizedNormalizedWeight\in\mxFormatSet^{\mxLayerWeightCount}\), define
\[
\mxScaledSurrogate(\mxQuantizedNormalizedWeight;\mxWeight)
=\tfrac12(\mxWeight-\mxFixedScaleMatrix\mxQuantizedNormalizedWeight)^\top
\mxTaskHessian(\mxWeight-\mxFixedScaleMatrix\mxQuantizedNormalizedWeight).
\]
Then every \(\mxRangeWidth\ge2\mxMaxGridGap\) satisfies
\[
\begin{aligned}
\min_{\mxQuantizedNormalizedWeight\in\mxFormatSet^{\mxLayerWeightCount}}
\mxScaledSurrogate(\mxQuantizedNormalizedWeight;\mxWeight)=
\min_{\mxQuantizedNormalizedWeight\in
\prod_{i=1}^{\mxLayerWeightCount}\mxCandidateSet_{\mxRangeWidth}(\mxNormalizedWeight_i)}
\mxScaledSurrogate(\mxQuantizedNormalizedWeight;\mxWeight).
\end{aligned}
\]
For E2M1 element codes, \(\mxRangeWidth=4\) suffices in normalized coordinates.
\end{corollary}
\begin{proof}
Since \(\mxWeight=\mxFixedScaleMatrix\mxNormalizedWeightVector\),
\[
\mxScaledSurrogate(\mxQuantizedNormalizedWeight;\mxWeight)
=\tfrac12(\mxNormalizedWeightVector-\mxQuantizedNormalizedWeight)^\top
\mxNormalizedHessian
(\mxNormalizedWeightVector-\mxQuantizedNormalizedWeight).
\]
Apply the bounded-window argument of
Proposition~\ref{prop:rangeround_window} to
\(\mxNormalizedWeightVector\) and \(\mxNormalizedHessian\).
As shown above, this argument requires no uniqueness assumption.
\end{proof}
The Hessian condition must be checked after transformation: unequal block
scales need not preserve factor-two row diagonal dominance.
The window width \(\mxRangeWidth\) is measured in normalized element coordinates.
}

\subsection{INT4--FP4 Reconstruction Example}
\label{sec:rounding_reconstruction_example}
We use \(\mxWeight=(1.4,4.9)^\top\) and the illustrative quadratic task-loss Hessian in Eq.~\eqref{eq:fp4_coupled_reconstruction}, with diagonal entries \(1\) and off-diagonal entries \(0.5\). Its eigenvalues are \(0.5\) and \(1.5\), and it satisfies the factor-two row condition at equality. FP4 uses unit scale and the distinct E2M1 values \(\mxFpGrid=\{0,\pm0.5,\pm1,\pm1.5,\pm2,\pm3,\pm4,\pm6\}\). INT4 uses \(\mxUniformStep=6/7\), matching the positive endpoint \(6\); its negative endpoint is \(-48/7\). Both scales remain fixed. At unit scale, normalized and physical weights coincide. The figure loss label \(\mxFigureLossLabel\) denotes the surrogate \(\mxExampleLoss\).

For the FP4 window \(\mxRangeWidth=4\), the exact candidate sets are
\begin{equation*}
\begin{aligned}
\mxCandidateSet_4(1.4)&=\{-0.5,0,0.5,1,1.5,2,3\},\\
\mxCandidateSet_4(4.9)&=\{3,4,6\}.
\end{aligned}
\end{equation*}
Their Cartesian product \(\mxFpCandidateProduct(\mxWeight)=\prod_{i=1}^{\mxLayerWeightCount}\mxCandidateSet_4(\mxWeightScalar_i)\), labeled \(\mxWindowFigureLabel\) in the figures, contains 21 legal pairs, shown in teal in Figure~\ref{fig:int4_fp4_reconstruction_comparison}. The strict-neighbor rule is the same as in the method; \(\mxRangeWidth\) is a window width, not a count of levels.

\begin{table}[t]
\centering
\caption{Enumeration of all legal pairs and restricted candidate sets for \(\mxWeight=(1.4,4.9)^\top\) and coupling \(0.5\). INT4 losses are rounded to six decimal places.}
\label{tab:rounding_reconstruction_example}
\begin{tabular}{llcc}
\toprule
Format & Assignment rule & Weight pair & Loss\\
\midrule
FP4 & \agentsreplace{best up--down}{Best up--down} & \((1.5,4)\) & 0.365\\
FP4 & \agentsreplace{best Full grid}{Full grid} & \((2,4)\) & 0.315\\
\midrule
INT4 & Nearest & \((12/7,36/7)\) & 0.117041\\
INT4 & Best up--down & \((6/7,36/7)\) & 0.110918\\
INT4 & Full grid & \((6/7,36/7)\) & 0.110918\\
\bottomrule
\end{tabular}%
\end{table}
\begin{figure}
\centering
\includesvg[width=\columnwidth,inkscapelatex=false]{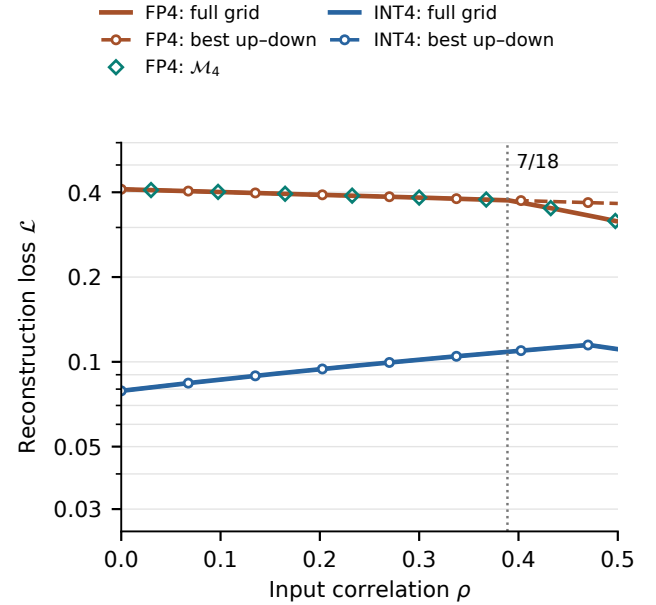}
\caption{Correlation sweep over the admissible interval \([0,0.5]\) for the fixed example, with one logarithmic loss axis. FP4 full-grid and window minima coincide, while INT4 full-grid and best up--down minima coincide. The vertical reference line marks the FP4 improvement threshold \(7/18\). Minima are computed by exhaustive enumeration.}
\label{fig:int4_fp4_correlation_log}
\end{figure}
For FP4, nearest rounding and best up--down coincide at \(\mxExampleAssignment{B}=(1.5,4)^\top\). The full-grid and window optima coincide at \(\mxExampleAssignment{C}=(2,4)^\top\). Moving from B to C increases squared weight error from \(0.82\) to \(1.17\), but reduces reconstruction loss from \(0.365\) to \(0.315\). INT4 nearest rounding is a different assignment from its best up--down and full-grid optimum, with respective losses \(1147/9800\) and \(1087/9800\). INT4 has lower absolute loss here; the comparison illustrates candidate-set sufficiency, rather than an FP4 accuracy advantage.

\begin{figure*}
\centering
\includesvg[width=0.96\textwidth,inkscapelatex=false]{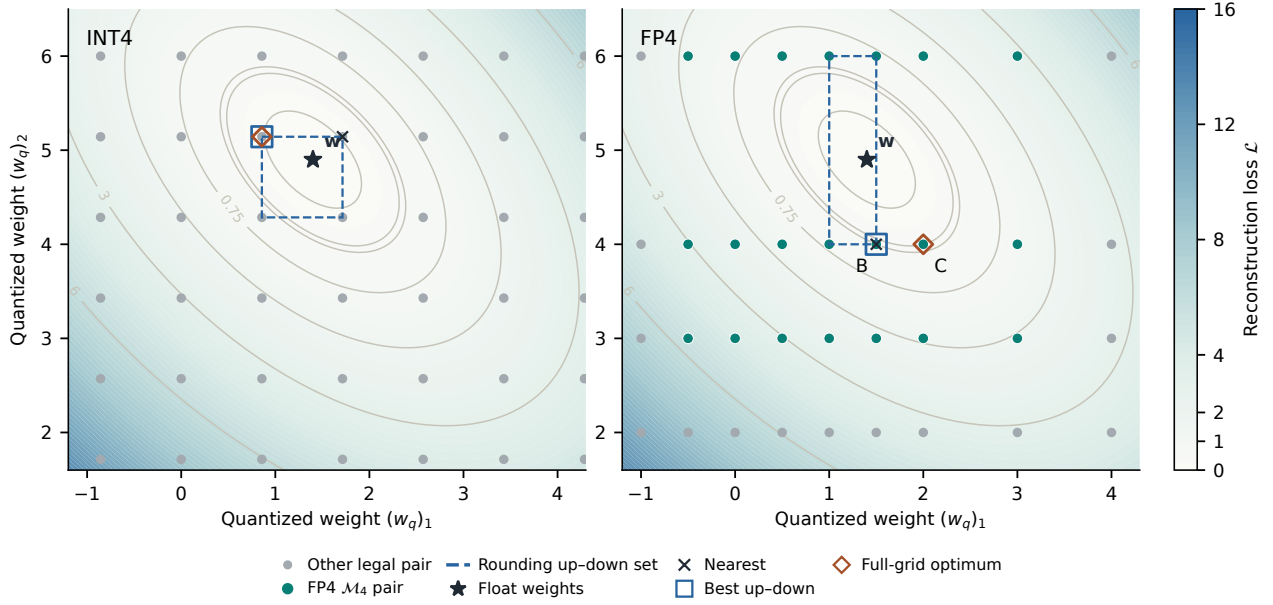}
\caption{INT4 and FP4 reconstruction for the same weights and matrix satisfying the factor-two row condition. The dashed blue rectangles enclose the rounding up--down candidates. In the FP4 panel, teal points mark all 21 candidates in \(\mxFpCandidateProduct(\mxWeight)\), including points beneath solution markers. Nested open markers indicate coincident assignments. The full-grid FP4 optimum is outside the rounding up--down set and inside the window; the INT4 full-grid optimum remains within the rounding up--down set. Both panels use the same axis limits and loss contours.}
\label{fig:int4_fp4_reconstruction_comparison}
\end{figure*}

\subsection{Effect of Input Correlation}
\label{sec:rounding_input_correlation}

We keep the weights and grids fixed and vary the off-diagonal entry \(\mxInputCorrelation\in[0,0.5]\), with unit diagonal entries. For a linear squared-error task with zero-mean, unit-variance inputs, \(\mxInputCorrelation\) is their correlation. Every matrix in this interval is positive definite and satisfies the factor-two row condition. Figure~\ref{fig:int4_fp4_correlation_log} shows all losses on one logarithmic axis. We enumerate the full grids, rounding up--down sets, and FP4 window product at each plotted correlation.

The FP4 rounding up--down pair \(\mxExampleAssignment{B}=(1.5,4)^\top\) has loss \(0.41-0.09\mxInputCorrelation\) and is best among the four corners throughout this interval. The additional pair \(\mxExampleAssignment{C}=(2,4)^\top\) has loss \(0.585-0.54\mxInputCorrelation\), so
\begin{equation}
\begin{aligned}
&\mxExampleLoss(\mxExampleAssignment{C};\mxWeight)
-\mxExampleLoss(\mxExampleAssignment{B};\mxWeight)\\
&\qquad=0.175-0.45\mxInputCorrelation.
\end{aligned}
\label{eq:fp4_example_correlation_threshold}
\end{equation}
The outside pair improves on every rounding up--down assignment when \(\mxInputCorrelation>7/18\). Enumeration confirms that the full-grid and window minima coincide throughout the sweep, as guaranteed by Proposition~\ref{prop:rangeround_window}.

For INT4, the loss of \((6/7,36/7)^\top\) minus that of the nearest pair \((12/7,36/7)^\top\) is \(\mxUniformStep^2(8-17\mxInputCorrelation)/60\). The two assignments tie at \(\mxInputCorrelation=8/17\); below this threshold the nearest pair is optimal, and above it the other pair is optimal. Full-grid and best up--down losses coincide throughout the enumerated sweep, including this tie. The theorem's uniqueness hypothesis does not apply at the tie, which is verified directly by enumeration. Neither switching threshold is asserted to hold for other weight pairs.

\section{MX Block Drop}\label{apx:block_drop}
\label{sec:mx_block_drop}

We adapt \qdrop{}~\cite{wei2022qdrop} to the MX block structure.
Whereas \qdrop{} samples an independent quantization-bypass decision for each activation element, we sample one decision for all elements sharing a block scale:
\begin{equation}
\begin{aligned}
\widetilde{\mxActivationScalar}_i
&=
{\color{red}m_{\mxBlockIndex}}\,\mxActivationScalar_i
+
\bigl(1-{\color{red}m_{\mxBlockIndex}}\bigr)
\bigl[{\color{red}\mxQuantizer}(\mxActivationTensor)\bigr]_i,
\\
&\qquad {\color{red}i\in\mxBlockElementSet_{\mxBlockIndex}},
\quad m_{\mxBlockIndex}\sim\operatorname{Bernoulli}(0.5).
\end{aligned}
\end{equation}
Here, $m_{\mxBlockIndex}=1$ retains the floating-point values of the entire MX block.
We apply this adaptation at the reconstruction-unit input and internal activation quantizers in both the \qdrop{} baseline and \vismixBW{}; \vismixLW{} disables it.
Inference uses fully quantized activations.
The red terms highlight the differences from elementwise \qdrop{}.

\clearpage
\onecolumn
\section{Additional Results}

\subsection{Additional Results on Vision Transformers}
\label{sec:additional_vision_transformer_results}
Here we present additional ImageNet-1K classification results for vision transformers across three 4-bit block-scaled formats.
\begin{table}[H]

\centering
\color{black}
\def\agentsedit#1{#1}
\def\agentsreplace#1#2{#2}
\renewcommand{\result}[2]{#1\,\ensuremath{\pm}\,#2}
\caption{ImageNet-1K Top-1 classification accuracy (\%) for ViT-family
models. Values with uncertainty report the mean and sample standard deviation
over seeds 0--4; RTN is deterministic and uses seed 0 only. MR-GPTQ
(FP4 block32) and SpinQuant (FP4 full) are joint weight-and-activation
results without reported seed variation. \agentsedit{\vismix{}} results are
retained from the original table. Within each format and network, darker blue marks the best result and light blue marks the second-best result.}
\label{tab:imagenet1k_classification_vit}
\begin{tabular}{llcccc}
\toprule
Format & Method &
\multicolumn{1}{c}{Tiny} &
\multicolumn{2}{c}{Middle} &
\multicolumn{1}{c}{Large} \\
\cmidrule(lr){3-3}
\cmidrule(lr){4-5}
\cmidrule(lr){6-6}
& &
DeiT-T &
ViT-S/16 & Swin-T &
DeiT-B \\
& &
22.90 MB &
88.20 MB & 114.30 MB &
346.30 MB \\
\midrule
Float & TBD &
72.17 &
81.38 & 81.38 &
81.98 \\
\midrule

\multirow{7}{*}{\shortstack{MXFP4\\\agentsreplace{E8M0}{UE8M0}}}
& RTN &
63.54 &
75.12 & 76.97 &
79.71 \\

& AdaRound &
\result{61.69}{0.06} &
\result{73.07}{0.09} & \result{77.96}{0.07} &
\result{78.99}{0.02} \\

& \agentsedit{\qdrop{}} &
\cellcolor[HTML]{DDEBF7}\color{black} \result{66.50}{0.11} &
\result{75.88}{0.08} & \cellcolor[HTML]{DDEBF7}\color{black} \result{79.40}{0.07} &
\result{80.34}{0.03} \\

& MR-GPTQ &
66.498 &
\cellcolor[HTML]{DDEBF7}\color{black} 77.844 & 78.118 &
\cellcolor[HTML]{DDEBF7}\color{black} 80.396 \\

& SpinQuant &
66.396 &
\cellcolor[HTML]{9DC3E6}\color{black} 77.902 & 78.292 &
\cellcolor[HTML]{9DC3E6}\color{black} 80.490 \\

& \textbf{\agentsedit{\vismixLW{}} (Ours)} &
\result{61.88}{0.18} &
\result{72.07}{0.08} & \result{77.96}{0.06} &
\result{79.42}{0.08} \\

& \textbf{\agentsedit{\vismixBW{}} (Ours)} &
\cellcolor[HTML]{9DC3E6}\color{black} \result{67.37}{0.13} &
\result{73.41}{0.20} & \cellcolor[HTML]{9DC3E6}\color{black} \result{79.45}{0.06} &
\result{80.30}{0.05} \\
\midrule

\multirow{7}{*}{\shortstack{MXFP4\\UE5M3}}
& RTN &
67.29 &
78.13 & 78.87 &
80.62 \\

& AdaRound &
\result{65.82}{0.06} &
\result{77.16}{0.10} & \result{79.14}{0.05} &
\result{80.45}{0.04} \\

& \agentsedit{\qdrop{}} &
\cellcolor[HTML]{DDEBF7}\color{black} \result{68.43}{0.08} &
\result{78.10}{0.07} & \cellcolor[HTML]{DDEBF7}\color{black} \result{79.98}{0.07} &
\result{80.85}{0.05} \\

& MR-GPTQ &
67.952 &
\cellcolor[HTML]{DDEBF7}\color{black} 79.162 & 79.310 &
80.752 \\

& SpinQuant &
68.058 &
\cellcolor[HTML]{9DC3E6}\color{black} 79.240 & 79.170 &
80.738 \\

& \textbf{\agentsedit{\vismixLW{}} (Ours)} &
\result{66.31}{0.16} &
\result{76.09}{0.06} & \result{79.54}{0.05} &
\cellcolor[HTML]{9DC3E6}\color{black} \result{80.90}{0.05} \\

& \textbf{\agentsedit{\vismixBW{}} (Ours)} &
\cellcolor[HTML]{9DC3E6}\color{black} \result{69.33}{0.10} &
\result{76.87}{0.05} & \cellcolor[HTML]{9DC3E6}\color{black} \result{80.28}{0.04} &
\cellcolor[HTML]{DDEBF7}\color{black} \result{80.89}{0.05} \\
\midrule

\multirow{5}{*}{NVFP4}
& RTN &
67.98 &
\cellcolor[HTML]{9DC3E6}\color{black} 78.74 & 79.07 &
80.80 \\

& AdaRound &
\result{67.01}{0.11} &
\result{77.76}{0.06} & \result{79.46}{0.03} &
\result{80.62}{0.03} \\

& \agentsedit{\qdrop{}} &
\cellcolor[HTML]{DDEBF7}\color{black} \result{69.15}{0.09} &
\cellcolor[HTML]{DDEBF7}\color{black} \result{78.62}{0.09} & \cellcolor[HTML]{DDEBF7}\color{black} \result{80.14}{0.05} &
\cellcolor[HTML]{DDEBF7}\color{black} \result{81.01}{0.05} \\

& \textbf{\agentsedit{\vismixLW{}} (Ours)} &
\result{67.04}{0.08} &
\result{77.12}{0.02} & \result{79.68}{0.05} &
\result{80.99}{0.06} \\

& \textbf{\agentsedit{\vismixBW{}} (Ours)} &
\cellcolor[HTML]{9DC3E6}\color{black} \result{69.74}{0.10} &
\result{77.96}{0.04} & \cellcolor[HTML]{9DC3E6}\color{black} \result{80.31}{0.05} &
\cellcolor[HTML]{9DC3E6}\color{black} \result{81.17}{0.05} \\
\bottomrule
\end{tabular}%

\end{table}
Blockwise \vismix{} achieves the highest reported accuracy on DeiT-T and Swin-T across all three formats.

\subsection{Low-Light Image Enhancement Results}
\label{sec:additional_lowlight_results}

We evaluate RetinexFormer~\cite{cai2023retinexformer} and HVI-CIDNet~\cite{yan2025hvi} on all 100 test pairs of LOLv2-Real using MXFP4-UE8M0, MXFP4-UE5M3, and NVFP4. Table~\ref{tab:lolv2_real_image_enhancement} reports RGB PSNR, RGB SSIM, and LPIPS-Alex.

\begin{table}[!htbp]
\centering
\color{black}
\def\agentsedit#1{#1}
\def\agentsreplace#1#2{#2}
\renewcommand{\result}[2]{#1\,\ensuremath{\pm}\,#2}
\caption{W4A4 post-training quantization results on all 100 LOLv2-Real test pairs for RetinexFormer and HVI-CIDNet. Higher RGB PSNR and SSIM and lower LPIPS-Alex are better. SSIM and LPIPS-Alex values are multiplied by 100. Within each format, network, and metric, darker blue marks the best result and light blue marks the second-best result.}
\label{tab:lolv2_real_image_enhancement}
\resizebox{\textwidth}{!}{%
\begin{tabular}{llcccccc}
\toprule
Format & Method &
\multicolumn{3}{c}{RetinexFormer} &
\multicolumn{3}{c}{HVI-CIDNet} \\
\cmidrule(lr){3-5}\cmidrule(lr){6-8}
& &
PSNR $\uparrow$ &
SSIM $\times 100$ $\uparrow$ &
LPIPS $\times 100$ $\downarrow$ &
PSNR $\uparrow$ &
SSIM $\times 100$ $\uparrow$ &
LPIPS $\times 100$ $\downarrow$ \\
\midrule

Float & -- &
22.80 & 85.43 & 17.08 &
23.89 & 87.92 & 12.24 \\
\midrule

\multirow{5}{*}{\shortstack{MXFP4\\UE8M0}}
& RTN &
22.36 & 76.39 & 22.51 &
20.48 & 73.02 & 26.76 \\

& AdaRound &
\cellcolor[HTML]{9DC3E6}\color{black} 22.59 & 77.52 & 22.70 &
\result{23.09}{0.06} & \result{77.94}{0.21} & \result{24.92}{0.33} \\

& \agentsedit{\qdrop{}} &
\cellcolor[HTML]{DDEBF7}\color{black} \result{22.55}{0.05} & \result{77.08}{0.12} & \result{22.67}{0.19} &
\result{22.93}{0.06} & \result{78.12}{0.05} & \result{24.11}{0.13} \\

& \textbf{\agentsedit{\vismixLW{}}} &
22.45 & \cellcolor[HTML]{9DC3E6}\color{black} 79.98 & \cellcolor[HTML]{DDEBF7}\color{black} 21.38 &
\cellcolor[HTML]{DDEBF7}\color{black} \result{23.33}{0.03} & \cellcolor[HTML]{9DC3E6}\color{black} \result{82.39}{0.06} & \cellcolor[HTML]{9DC3E6}\color{black} \result{18.24}{0.19} \\

& \textbf{\agentsedit{\vismixBW{}}} &
\result{22.50}{0.05} & \cellcolor[HTML]{DDEBF7}\color{black} \result{79.57}{0.16} & \cellcolor[HTML]{9DC3E6}\color{black} \result{21.28}{0.20} &
\cellcolor[HTML]{9DC3E6}\color{black} \result{23.44}{0.09} & \cellcolor[HTML]{DDEBF7}\color{black} \result{81.85}{0.10} & \cellcolor[HTML]{DDEBF7}\color{black} \result{18.83}{0.06} \\
\midrule

\multirow{5}{*}{\shortstack{MXFP4\\UE5M3}}
& RTN &
22.42 & 78.37 & 21.03 &
23.32 & 79.87 & 20.05 \\

& AdaRound &
\cellcolor[HTML]{DDEBF7}\color{black} 22.70 & 80.27 & 20.36 &
\result{23.57}{0.08} & \result{82.76}{0.09} & \result{18.63}{0.11} \\

& \agentsedit{\qdrop{}} &
\result{22.53}{0.04} & \result{79.43}{0.05} & \result{20.43}{0.08} &
\result{23.43}{0.06} & \result{82.16}{0.08} & \result{19.24}{0.17} \\

& \textbf{\agentsedit{\vismixLW{}}} &
\cellcolor[HTML]{9DC3E6}\color{black} 22.74 & \cellcolor[HTML]{9DC3E6}\color{black} 81.52 & \cellcolor[HTML]{9DC3E6}\color{black} 20.11 &
\cellcolor[HTML]{9DC3E6}\color{black} \result{23.72}{0.05} & \cellcolor[HTML]{9DC3E6}\color{black} \result{85.59}{0.06} & \cellcolor[HTML]{DDEBF7}\color{black} \result{13.93}{0.06} \\

& \textbf{\agentsedit{\vismixBW{}}} &
\result{22.67}{0.03} & \cellcolor[HTML]{DDEBF7}\color{black} \result{80.77}{0.05} & \cellcolor[HTML]{DDEBF7}\color{black} \result{20.12}{0.04} &
\cellcolor[HTML]{DDEBF7}\color{black} \result{23.68}{0.03} & \cellcolor[HTML]{DDEBF7}\color{black} \result{85.40}{0.05} & \cellcolor[HTML]{9DC3E6}\color{black} \result{13.72}{0.09} \\
\midrule

\multirow{5}{*}{NVFP4}
& RTN &
22.51 & 79.95 & 20.13 &
23.38 & 82.59 & 17.68 \\

& AdaRound &
\cellcolor[HTML]{9DC3E6}\color{black} 22.66 & 80.79 & \cellcolor[HTML]{DDEBF7}\color{black} 19.84 &
\result{23.58}{0.02} & \result{83.37}{0.07} & \result{17.66}{0.11} \\

& \agentsedit{\qdrop{}} &
\result{22.55}{0.03} & \result{79.91}{0.01} & \result{20.01}{0.04} &
\result{23.59}{0.09} & \result{82.91}{0.17} & \result{18.34}{0.17} \\

& \textbf{\agentsedit{\vismixLW{}}} &
\cellcolor[HTML]{DDEBF7}\color{black} 22.65 & \cellcolor[HTML]{9DC3E6}\color{black} 82.04 & \cellcolor[HTML]{9DC3E6}\color{black} 19.22 &
\cellcolor[HTML]{9DC3E6}\color{black} \result{23.74}{0.01} & \cellcolor[HTML]{9DC3E6}\color{black} \result{85.03}{0.03} & \cellcolor[HTML]{9DC3E6}\color{black} \result{14.67}{0.07} \\

& \textbf{\agentsedit{\vismixBW{}}} &
\result{22.49}{0.02} & \cellcolor[HTML]{DDEBF7}\color{black} \result{81.27}{0.11} & \cellcolor[HTML]{DDEBF7}\color{black} \result{19.84}{0.17} &
\cellcolor[HTML]{DDEBF7}\color{black} \result{23.61}{0.06} & \cellcolor[HTML]{DDEBF7}\color{black} \result{84.50}{0.14} & \cellcolor[HTML]{DDEBF7}\color{black} \result{15.02}{0.16} \\
\bottomrule
\end{tabular}%
}
\end{table}

\subsection{\agentsreplace{Weight / Activation Analysis}{Weight and Activation Analysis}}
\label{sec:additional_weight_activation_results}
Here we present weight-only and activation-only MXFP4 quantization results to examine their separate effects across vision models and tasks.


\begin{table}[H]
\centering
\caption{RTN single-side ablation using MXFP4-\agentsreplace{E8M0}{UE8M0}. Higher is better for all metrics.}
\label{tab:rtn-single-side-ablation}
\begin{tabular}{lllrrr}
\toprule
Task & Model & Metric & Float & Weight-only & Activation-only \\
\midrule
Classification & MobileNetV2          & Top-1 (\%) & 72.15 & 1.86  & 34.43 \\
Classification & DeiT-T               & Top-1 (\%) & 72.17 & 69.00 & 67.08 \\
Classification & ResNet-18            & Top-1 (\%) & 69.76 & 63.14 & 63.29 \\
Classification & CaiT-XXS24           & Top-1 (\%) & 78.39 & 74.85 & 74.04 \\
Classification & ConvNeXt-Tiny        & Top-1 (\%) & 82.52 & 81.30 & 78.63 \\
Classification & EfficientFormer-L3   & Top-1 (\%) & 82.55 & 75.46 & 76.85 \\
Classification & ViT-S/16             & Top-1 (\%) & 81.38 & 79.07 & 78.39 \\
Classification & Swin-T               & Top-1 (\%) & 81.38 & 79.40 & 78.69 \\
Classification & ResNeXt-101 $32{\times}8$d
                                         & Top-1 (\%) & 82.83 & 75.30 & 73.78 \\
Classification & EfficientViT-L2      & Top-1 (\%) & 85.06 & 0.10  & 0.10 \\
Classification & DeiT-B               & Top-1 (\%) & 81.98 & 81.31 & 80.39 \\
\midrule
Object detection & Faster R-CNN R50-FPN & bbox mAP (\%) & 36.92 & 31.35 & 27.94 \\
Object detection & RetinaNet R50-FPN     & bbox mAP (\%) & 36.35 & 30.56 & 27.38 \\
Object detection & FCOS R50-FPN          & bbox mAP (\%) & 39.11 & 32.37 & 29.94 \\
Object detection & YOLOv8-s              & bbox mAP (\%) & 39.62 & 7.47  & 19.77 \\
\midrule
Segmentation & DeepLabV3 R50            & mIoU (\%) & 66.22 & 64.06 & 59.46 \\
Segmentation & DeepLabV3 MobileNetV3    & mIoU (\%) & 60.04 & 6.35  & 19.62 \\
\midrule
Low light & RetinexFormer & SSIM & 0.8543 & 0.8493 & 0.7651 \\
Low light & HVI-CIDNet    & SSIM & 0.8792 & 0.8605 & 0.7520 \\
\bottomrule
\end{tabular}%
\end{table}

Weight-only quantization causes greater degradation than activation-only quantization in MobileNetV2 and DeepLabV3-MobileNetV3. The opposite pattern holds for the evaluated low-light models, showing that the dominant source of degradation depends on the model and task.

\subsection{MXFP8 Results}
\label{sec:additional_mxfp8_results}
Here we present MXFP8 results using UE5M3 and UE8M0 block scales, with and without bias correction.

\begin{table}[H]

\centering
\caption{MXFP8 E4M3 RTN results using UE5M3 and \agentsreplace{E8M0}{UE8M0} scales. \agentsreplace{Classification = Top-1 \%, detection = bbox mAP \%, segmentation = mIoU \%, low-light = PSNR dB.}{The metrics are top-1 accuracy (\%) for classification, bounding-box mAP (\%) for detection, mIoU (\%) for segmentation, and PSNR (dB) for low-light enhancement.}}
\label{tab:mxfp8_rtn_bias_correction}
\begin{tabular}{llrrrrr}
\toprule
& & & \multicolumn{2}{c}{UE5M3 scale} & \multicolumn{2}{c}{\agentsreplace{E8M0}{UE8M0} scale} \\
\cmidrule(lr){4-5}\cmidrule(lr){6-7}
Task & Model & Float & RTN w/o BC & RTN w/ BC & RTN w/o BC & RTN w/ BC \\
\midrule
Classification & MobileNetV2
& 72.15 & 69.44 & 70.95
& 64.60 & 69.81 \\
Classification & DeiT-T
& 72.17 & 71.87 & 71.84
& 71.79 & 71.78 \\
Classification & ResNet-18
& 69.76 & 69.16 & 69.24
& 69.17 & 69.35 \\
Classification & CaiT-XXS24
& 78.39 & 78.18 & 78.17
& 77.97 & 77.95 \\
Classification & ConvNeXt-Tiny
& 82.52 & 82.23 & 82.28
& 82.17 & 82.20 \\
Classification & EfficientFormer-L3
& 82.55 & 82.33 & 82.35
& 82.12 & 82.14 \\
Classification & ViT-S/16
& 81.38 & 81.30 & 81.33
& 81.28 & 81.31 \\
Classification & Swin-T
& 81.38 & 81.16 & 81.15
& 81.05 & 81.08 \\
Classification & ResNeXt-101-32x8d
& 82.83 & 82.27 & 82.37
& 82.29 & 82.27 \\
Classification & EfficientViT-L2
& 85.06 & 56.44 & 76.49
& 0.41 & 2.36 \\
Classification & DeiT-B
& 81.98 & 81.83 & 81.85
& 81.73 & 81.80 \\
\midrule
Object detection & Faster R-CNN ResNet50-FPN
& 36.92 & 36.26 & 36.24
& 36.14 & 36.27 \\
Object detection & RetinaNet ResNet50-FPN
& 36.35 & 35.69 & 35.74
& 35.55 & 35.63 \\
Object detection & FCOS ResNet50-FPN
& 39.11 & 38.44 & 38.51
& 38.30 & 38.27 \\
Object detection & YOLOv8s
& 39.62 & 38.25 & 38.29
& 36.50 & 37.52 \\
\midrule
Segmentation & DeepLabV3 ResNet50
& 66.22 & 66.04 & 66.13
& 65.63 & 65.74 \\
Segmentation & DeepLabV3 MobileNet-V3-Large
& 60.04 & 58.35 & 58.18
& 56.38 & 57.35 \\
\midrule
Low-light & RetinexFormer
& 22.80 & 22.79 & 22.81
& 22.85 & 22.81 \\
Low-light & HVI-CIDNet
& 23.89 & 23.78 & 23.86
& 23.61 & 23.73 \\
\bottomrule
\end{tabular}%
\end{table}


MXFP8 retains near-floating-point performance for many evaluated models, while EfficientViT-L2 remains highly sensitive to the scale format. Bias correction improves some sensitive configurations, but its benefits are not consistent across all models. Applying \textbf{VisionMX-BW{}} MXFP8-UE8M0 and MXFP8-UE5M3 on EfficientViT-L2 gets 84.93\% and 84.94\% top-1 accuracy, respectively (with $\delta=0.7$).

\subsection{\agentsedit{Layerwise Ablation without \qdrop{}}}
\label{sec:layerwise_ablation}
Here we present a layerwise ablation of RangeRound, learned bias updates, and AAC without \qdrop{}. Its optimization settings differ from the main-text blockwise ablation, so the two tables do not provide a matched comparison of reconstruction strategies.

\begin{table}[H]

\centering
\caption{AdaRound $\rightarrow$ \agentsedit{\vismixLW{}} ablation with MXFP4-UE8M0.
All rows use layerwise optimization without \agentsedit{\qdrop{}}, 20{,}000 steps,
2{,}048 calibration samples, warmup $0.20$, MSE loss, and
$\epsilon=10^{-8}$. MobileNetV2 uses rounding LR $3\times10^{-3}$,
rounding regularization $1.00$, bias LR $10^{-3}$, bias regularization
$0$, and RangeRound $\delta=3.00$. CaiT-XXS24 uses rounding LR
$10^{-3}$, rounding regularization $0.01$, bias LR $10^{-3}$, bias
regularization $0$, and RangeRound $\delta=2.00$. Top-1 accuracy is reported in percent.}
\label{tab:lw-fixed-column-ablation}

\begin{tabular}{@{}ccc rr@{}}
\toprule
RR & Bias & \agentsreplace{Affine}{AAC} & MobileNetV2 & CaiT-XXS24 \\
\midrule
\ding{55} & \ding{55} & \ding{55} & 41.80 & 53.86 \\
\ding{55} & \ding{55} & \ding{51} & 64.07 & 60.90 \\
\ding{55} & \ding{51} & \ding{55} & 49.47 & 55.45 \\
\ding{55} & \ding{51} & \ding{51} & 65.46 & 63.89 \\
\ding{51} & \ding{55} & \ding{55} & 47.34 & 59.18 \\
\ding{51} & \ding{55} & \ding{51} & 64.25 & 62.49 \\
\ding{51} & \ding{51} & \ding{55} & 50.06 & 60.67 \\
\ding{51} & \ding{51} & \ding{51} & 65.11 & 66.21 \\
\bottomrule
\end{tabular}%
\end{table}

Combining all three components yields 65.11\% on MobileNetV2 and 66.21\% on CaiT-XXS24. This is the highest reported accuracy for CaiT-XXS24; for MobileNetV2, the highest accuracy is 65.46\%, obtained with bias updates and AAC but without RangeRound.

\end{document}